\documentclass[lettersize,journal]{IEEEtran}
\usepackage{amsmath,amsfonts}
\usepackage{algorithmic}
\usepackage{algorithm}
\usepackage{array}
\usepackage[caption=false,font=normalsize,labelfont=sf,textfont=sf]{subfig}
\usepackage{textcomp}
\usepackage{stfloats}
\usepackage{url}
\usepackage{verbatim}
\usepackage{graphicx}
\usepackage{cite}
\usepackage{orcidlink}

\usepackage{latexsym}
\usepackage{amssymb}
\usepackage{amsmath}
\usepackage{amsthm}
\usepackage{booktabs}
\usepackage{enumitem}
\usepackage{graphicx}
\usepackage{color}
\graphicspath{{./}}
\newtheorem{theorem}{Theorem}
\newtheorem{lemma}[theorem]{Lemma}

\newtheorem{proposition}[theorem]{Proposition}

\newcommand{\BibTeX}{B\kern-.05em{\sc i\kern-.025em b}\kern-.08em\TeX}

\usepackage{bm}
\usepackage{algorithm}
\usepackage{algorithmic}

\newcommand{\LineComment}[1]{\STATE \(//\) \texttt{#1}}

\usepackage{multirow}
\usepackage{threeparttable}
\usepackage{mfirstuc}

\def\entitle{Does Machine Bring in Extra Bias in Learning? Approximating Fairness in Models Promptly}
\def\entitletwo{Measuring Model-Induced Discrimination via Efficient Fairness Approximation}

\DeclareMathOperator*{\argmin}{argmin}

\newcommand{\defineq}{\triangleq}

\def\xneg{\mathop{\breve{\bm{x}}}}
\def\xpos{\mathop{\bm{a}}}
\def\xqtb{\mathop{\tilde{\bm{a}}}}

\def\probP{\mathbb{P}}
\def\insx{\bm{x}}

\def\vecw{\bm{w}}

\def\volume{\mathrm{Vol}}
\def\ball{\mathbf{B}}%
\def\volball{\mathrm{Vol}}

\def\newy{\ddot{y}}
\def\comp{\mathcal{O}}

\def\newDist{\mathbf{D}}
\def\newdist{\mathbf{d}}
\def\newFist{\mathbf{D}_f}
\def\newfist{\mathbf{df}}

\def\ppsdisfull{harmonic fairness via manifolds}
\def\ppsdisabbr{\emph{HFM}}
\def\ppsalgfull{Approximation of distance between sets}
\def\ppsalgabbr{\emph{ApproxDist}}
\def\ppssubfull{Acceleration sub-procedure in approximation}
\def\ppssubabbr{\emph{AcceleDist}}

\newcommand{\topequation}{%
  \setlength\abovedisplayskip{1pt}%
  \setlength\belowdisplayskip{1pt}%
  \small%
}

\def\ie{i.e.}
\def\aka{aka.}

\usepackage{stmaryrd}
\newcommand{\citet}[1]{\citeauthor{#1} \shortcite{#1}}
\newcommand{\citep}{\cite}

\usepackage{xfrac}
\usepackage{soul}
\usepackage{xcolor}

\newcommand{\myref}[2]{\ref{#1}\subref{#2}}

\usepackage[acronym,shortcuts,xindy]{glossaries}

\usepackage[acronym,shortcuts,xindy]{glossaries}

\newacronym{ai}{AI}{artificial intelligence}
\newacronym{ml}{ML}{machine learning}
\newacronym{dl}{DL}{deep learning}
\newacronym{fml}{FairML}{fairness in machine learning}
\newacronym{sa}{SA}{sensitive attribute}

\newacronym{dp}{DP}{demographic parity}
\newacronym{eo}{EOdd}{equalised odds}
\newacronym{eopp}{EO}{equality of opportunity}
\newacronym{pqp}{PQP}{predictive parity}
\newacronym{cff}{CFF}{counterfactual fairness}
\newacronym{sp}{SP}{statistical parity}

\newacronym{dr}{DR}{discriminative risk}
\newacronym{hfm}{HFM}{harmonic fairness via manifolds}

\begin{document}

\title{\entitletwo}
\author{Yijun Bian$^*$ and Yujie Luo$^*$%
\thanks{Manuscript received August 17, 2024; revised September 18, 2025; revised February 2, 2026; revised April 5, 2026; accepted June 17, 2026. 
(Correspondence to Yijun Bian.)}%
\thanks{$^*$These authors contributed equally.}%
\thanks{Y. Bian is with the Department of Computer Science, University of Copenhagen, 2100 Copenhagen, Denmark (e-mail: yibi@di.ku.dk).}%
\thanks{Y. Luo is with the Department of Mathematics, National University of Singapore, Singapore 117543, Singapore (e-mail: lyj96@nus.edu.sg).}%
}
\markboth{IEEE Transactions on Neural Networks and Learning Systems,~Vol.~, No.~, Month~Year}
{Anonymous \MakeLowercase{\textit{et al.}}: \entitle}

\maketitle

\begin{abstract}
Providing various machine learning (ML) applications in the real world, concerns about discrimination hidden in ML models are growing, particularly in high-stakes domains. 
Existing techniques for assessing the discrimination level of ML models include commonly used group and individual fairness measures. 
However, these two types of fairness measures are usually hard to be compatible, 
and even two different group fairness measures might be incompatible as well. 
To address this issue, we investigate and evaluate the discrimination level of classifiers from a manifold perspective and propose a fairness measure named ``\emph{\ppsdisfull{} (}\ppsdisabbr\emph{)}'' based on distances between sets. 
Yet the direct calculation of distances might be too expensive to afford, reducing its practical applicability. 
Therefore, we devise an approximation algorithm named ``\emph{\ppsalgfull{} (}\ppsalgabbr\emph{)}'' to facilitate accurate estimation of distances, and we further demonstrate its algorithmic effectiveness under certain reasonable assumptions. 
Empirical results indicate that the proposed fairness measure \ppsdisabbr{} reflects bias from both individual and group fairness aspects and that the proposed \ppsalgabbr{} is effective and efficient. 
\end{abstract}

\begin{IEEEkeywords}
Machine learning, fairness measure, distance between sets, metric space, approximation.
\end{IEEEkeywords}

\section{Introduction}

Numerous real-world applications nowadays have been advanced to assist humans thanks to the flourishing development of machine learning (ML) techniques, such as healthcare, transportation, recruitment, and jurisdiction. 
For instance, many scholars have attempted to develop ML systems to alleviate resource constraints and reduce overall healthcare costs, such as the early diagnosis of Alzheimer's disease and automated detection of suspicious mammographic lesions for biopsy \citep{chua2022tackling}. 
However, deploying ML systems in high-risk scenarios raises concerns regarding unresolved safety issues. These issues include erroneous predictions, lack of confidence in predictions, imperfect generalisability, black-box decision-making, and insensitivity to impact and automation complacency \citep{chua2022tackling}. 
Among them, an area to investigate trustworthy AI is gradually gaining researchers' attention, for example, fairness in machine learning. 

The reason why people worry about the fairness and reliability of ML systems is that discriminative models may perpetuate or even exacerbate inappropriate human prejudices, which harm not only model performance but also society. 
Taking the ML systems deployed for clinical use, for example, 
shifted distributions exist within the collected medical datasets, thus causing prediction failures, as a higher incidence of a disease may be reported in people seeking medical treatment in hospitals, of which the characteristics are different from those of a dataset collected for population studies \citep{azizi2023robust}. 
Besides, potential sources of dataset bias also exist, such as imbalanced proportions for different groups, specific measurement biases captured via imaging devices, systematic biases introduced by human annotators, and distorted research subject to dataset availability. 
Last but not least, inappropriate evaluation for ML models may miss the desired target, and a strong focus on benchmark performance would lead to diminishing returns, where the performance gains of increasingly great efforts are actually smaller and smaller \citep{varoquaux2022machine}. It is worthwhile to choose suitable metrics in evaluation, as improvements in traditional metrics such as accuracy do not necessarily translate to practical clinical outcomes. 
Therefore, it is crucial in the responsible deployment of ML systems to evaluate model performance from both accuracy and fairness.

Unlike precise accuracy, however, the concept of ``fairness'' is interdisciplinary with broadly varying definitions across disciplines, such as law, politics, social sciences, quantitative fields (mathematics, computer science, statistics, economics), and philosophy. 
Even in the field of computer science that we mainly discuss in this paper, there exist so many fairness explanations or measures to choose from when applied to the deployment of ML systems, of which two typical categories are group- and individual-level fairness. 
While individual-level fairness follows the intuition that similar predicted outcomes are expected for similar \emph{individuals} (in other words, similar treatment) given some specific similarity metrics \citep{dwork2012fairness,fleisher2021s}, group-level fairness cares more about the \emph{statistical} disparity among groups divided by certain sensitive/unprivileged/protected characteristics, typically including age, race, gender, disability, pregnancy, and sexual orientation \citep{cdc2021types,eeoc2022type}. 
Typical group-level fairness measures include demographic parity (\aka{} statistical parity) \citep{calders2009building,dwork2012fairness,feldman2015certifying,zemel2013learning}, equalised odds \citep{hardt2016equality}, and predictive parity \citep{ieterich2016demonstrating}, and they respectively correspond to independence, separation, and sufficiency in scenarios for sensitive attributes with binary values \citep{barocas2023fairness}. 
However, it is hard to meet all three criteria (\aka{} independence, separation, and sufficiency) simultaneously, as imposing any two of them at the same time would overconstrain the space to the point where only degenerate solutions remain \citep{barocas2023fairness}. 
Moreover, satisfying group fairness such as statistical parity may violate individual fairness terribly \citep{dwork2012fairness,fleisher2021s,tang2023and}; the inherent trade-off between fairness and accuracy also draws researchers' attention \citep{zhao2022inherent}. 
Therefore, it is necessary to deliberate on the incompatibility of these criteria when evaluating ML model performance.

To this end, we investigate the possibility of assessing the discrimination level of ML models from both individual and group fairness aspects. 
In this paper, we propose a novel fairness measure from a manifold perspective, named ``\emph{\ppsdisfull{} (}\ppsdisabbr\emph{)}'', to fulfil this purpose. 
Intuitively, by viewing instances as points on data manifolds, \ppsdisabbr{} measures discrimination through the geometric discrepancy between manifolds induced by privileged and unprivileged groups.
However, the calculation of \ppsdisabbr{} lies on a core distance between two sets, which might be pretty costly. 
To speed up the calculation and increase its practical applicable values, we propose an approximation algorithm to quickly estimate the distance between sets, named as ``\emph{\ppsalgfull{} (}\ppsalgabbr\emph{)}.'' 
We also further investigate its algorithmic properties under certain reasonable assumptions, in other words, how effective it is to achieve the approximation goal.
Our contribution in this work is fourfold:
\begin{itemize}
\item We propose a fairness measure \ppsdisabbr{} that could reflect the discrimination level of classifiers from both individual and group fairness aspects, which is built on a concept of distances between sets. 
\item We propose an approximation algorithm named ``\ppsalgfull{} (\ppsalgabbr)'' to speed up the estimation of distances between sets, in order to mitigate the disadvantage of its costly direct calculation. 
\item We further investigate the algorithmic effectiveness of \ppsalgabbr{} under certain assumptions and provide detailed explanations. 
\item Comprehensive experiments are conducted to verify the effectiveness of the proposed \ppsdisabbr{} and \ppsalgabbr{}. 
\end{itemize}

\section{Related Work}
In this section, we first introduce the sources of biases and then summarise existing techniques to measure and enhance fairness in turn. 

\paragraph{Sources of biases} 
Before addressing fairness concerns, it is necessary to understand the origins of biases. 
Existing literature identifies two primary sources of unfairness: biases from the data and biases from the algorithm \citep{verma2018fairness}. 
Biased data collected from various sources, like device measurements and historically biased human decisions, directly influence ML algorithms, perpetuating these biases. 
Additionally, missing data, such as instances or values, can introduce disparities between the dataset and the target population, leading to biased outcomes. 
Even with clean data, learning algorithms might yield unfair results due to proxy attributes for sensitive features or tendentious algorithmic objectives. 
For example, optimising aggregated prediction errors can advantage privileged groups over marginalised ones.

\paragraph{Mechanisms to enhance fairness} 
Numerous mechanisms have been proposed to enhance fairness and mitigate biases in ML models, typically categorised as pre-processing, in-processing, and post-processing mechanisms \citep{pessach2023algorithmic,tang2023and,mehrabi2021survey}. 
Pre-processing mechanisms \citep{backurs2019scalable,chierichetti2017fair,calmon2017optimized,tian2024multifair} usually take place when the algorithm is allowed to modify the training data, manipulating features or labels before instances are fed into the algorithm, in order to align the distributions of unprivileged and privileged groups to minimise discrimination. 
For instance, recent scholarship utilises counterfactual samples \citep{yang2026circus,xia2025learning} or counterfactual-like data augmentation \citep{zhu2025fairshap} to promote fairness. 
In contrast, post-processing mechanisms \citep{dwork2018decoupled,corbett2017algorithmic,hardt2016equality} are normally used when the learned model can only be treated as a black box without any ability to modify the training data or learning algorithm, manipulating output scores or decoupling predictions for each group. 
While pre-processing and post-processing mechanisms offer flexibility across classification tasks, they have drawbacks. 
Pre-processing mechanisms suffer from high uncertainty in final accuracy, while post-processing mechanisms often yield inferior results due to the late-stage application in the learning process. 
Additionally, they may not fully achieve individual fairness. 
In-processing mechanisms \citep{zafar2017fairness2,zafar2017fairness1,woodworth2017learning,quadrianto2017recycling,du2024long} incorporate fairness constraints during training by using penalty/regularisation terms, while some of them \citep{agarwal2019fair,agarwal2018reductions} adjust fairness constraints in minimax or multi-objective optimisation settings.  
They usually impose explicit trade-offs between fairness and accuracy in the objective function, yet the designated ML algorithm itself is tightly coupled with them as well. 
It is worth noting that these approaches primarily focus on mitigation rather than measurement, and their effectiveness is relevant to the choice of the fairness measure that is applied. 
It is hard to say which one outperforms the others in all cases; the results depend on the fairness measures, datasets, and even the handling of training-test splits \citep{friedler2019comparative}.

\paragraph{Types of fairness measures} 
Various fairness measures have been proposed to facilitate the design of fair ML models, generally divided into distributive and procedural fairness measures. 
Procedural fairness concerns decision-making processes and includes feature-apriori fairness, feature-accuracy fairness, and feature-disparity fairness \citep{grgic2018beyond}. 
However, these measures are dependent on features and a set of users who view the corresponding features as fair ones and thus still risk hidden discrimination in the data. 
Moreover, these user judgements may change over time, resulting in unstable outcomes and unpredictable computational demands as systems are recalibrated repeatedly. 
In addition, an FAE-based (feature attribution explanation) metric has recently been proposed to assess group procedural fairness, which, however, depends on the specific FAE techniques employed \citep{wang2024procedural}; 
BENN \citep{giloni2022benn} is also proposed for bias estimation, but requires another unsupervised neural network to be trained, adding some computational burdens.

Distributive fairness refers to decision-making outcomes (predictions) and includes unconscious/unawareness fairness, group fairness, individual fairness, and counterfactual fairness \citep{dwork2012fairness, berk2021fairness,vzliobaite2017measuring, joseph2016fairness, nilforoshan2022causal,kusner2017counterfactual}. 
As the simplest, unawareness fairness avoids explicitly using sensitive attributes but does not address biases stemming from associations between protected and unprotected attributes. 
Group fairness focuses on statistical equality among groups defined by sensitive attributes, including but not limited to demographic parity, equalised odds, equality of opportunity/equal opportunity, predictive quality parity, and disparate impact \citep{feldman2015certifying}. 
In contrast, individual fairness is based on the principle that `similar individuals should be evaluated or treated similarly,' where similarity is measured by some certain distance between individuals, while the specified distance also matters a lot. 
Besides, counterfactual fairness aims to explain the sources of discrimination and qualitative equity through causal inference tools. 
Other distributive fairness definitions include fairness through unawareness \citep{dwork2012fairness,gajane2017formalizing}, calibration \citep{chouldechova2017fair}, multicalibration \citep{hebert2018multicalibration,gohar2023survey}, multiaccuracy \citep{kim2019multiaccuracy}, differentially fair and intersectionality \citep{foulds2020intersectional}, and strategic minimax fairness \citep{diana2024minimax}. Yet none of these are quantitative measures, except group benefit ratio and worst-case min-max ratio \citep{ghosh2021characterizing}. 

Moreover, the group fairness measures introduced in different types of fairness may hardly be compatible with each other, for example, the occurrence between equalised odds and demographic parity, or that between equalised calibration and equalised odds \citep{berk2021fairness,pleiss2017fairness,hardt2016equality}. 
Individual and group fairness, such as demographic parity, are also incompatible except in trivial degenerate cases. 
Meeting all three fairness criteria (\ie{} independence, separation, and sufficiency) simultaneously is demonstrated to be challenging, often only achievable in degenerate scenarios \citep{barocas2023fairness}. 
Recently, a fairness measure named `discriminative risk (DR)' has been proposed to evaluate the bias level of classifiers from both individual- and group-fairness aspects \citep{bian2023increasing_alt}, and has been incorporated into fairness-oriented pre-processing frameworks such as FairSHAP \citep{zhu2025fairshap}. 
Besides, achieving higher fairness can often come at the cost of compromised accuracy \citep{menon2018cost}. 
Although some unique scenarios \cite{pessach2021improving,wick2019unlocking} have recently been proposed where fairness and accuracy can be concurrently improved, constrained optimisation with additional fairness constraints likely results in reduced accuracy compared to accuracy-focused optimisation.

\section{Methodology}

In this section, we formally study the measurement of fairness from a manifold perspective. 
A list of some of the standard notations we use is summarised in Table~\ref{tab:notation}. 
\begin{table}[t]
\centering\caption{
Summary of the standard notations we use.
}\label{tab:notation}%
\vspace{-2mm}
\begin{tabular}{cl}
\toprule
\textbf{Notation} & \textbf{Implication} \\
\midrule
$x$ & scalars, denoted by italic lowercase letters \\
$\bm{x}$ & vectors, denoted by bold lowercase letters \\
$X$ & matrices/sets, denoted by italic uppercase letters \\
$\mathsf{X}$ & random variables, denoted by serif uppercase letters \\
$\mathbb{R}$ & real numbers \\
$\mathbb{Z}$ & integers \\
$\mathbb{Z}_+$ & positive integers \\
$\mathbb{P}(\cdot)$ & the probability measure of one random variable \\
$\mathbb{E}(\cdot)$ & the  expectation of one random variable \\
$\mathbb{V}(\cdot)$ & the variance of one random variable \\
$\mathcal{F}$ & the hypothesis space \\
$f(\cdot)$ & models in the space $\mathcal{F}$ \\
$i\in[n]$ & $i\in\{1,2,...,n\}$ for brevity \\
\bottomrule
\end{tabular}
\end{table}
We also use $S\!=\{(\bm{x}_i,y_i)\}_{i=1}^n$ to denote a dataset where the instances are iid. (independent and identically distributed), drawn from an feature/input-label/output space $\mathcal{X}\times\mathcal{Y}$ based on an unknown distribution. 
The feature space $\mathcal{X}$ is arbitrary, and the label space $\mathcal{Y}=\{ 1,2,...,n_c \} \,(n_c\geqslant 2)$ is finite, which could be binary or multi-class classification depending on the number of labels (\aka{} the value of $n_c$). 
Presuming that the considered dataset $S$ is composed of the instances including sensitive attributes, also denoted by $\{(\xneg_i,\xpos_i,y_i)\}_{i=1}^n$, the features of one instance including protected or sensitive attributes is represented as $\bm{x}\defineq (\xneg,\xpos)$, where for clarity $\xneg=[x_1,...,x_{n_x}]^\mathsf{T}$ is non-sensitive attributes (possibly involving proxy attributes though), $n_x\in\mathbb{Z}_+$ is the number of non-sensitive attributes, $\xpos=[a_1,..., a_{n_a}]^\mathsf{T}$ is the sensitive attributes, $n_a\in\mathbb{Z}_+ (n_a\geqslant 1)$ is the number of sensitive attributes allowing multiple attributes, and $a_i\in\mathbb{Z}_+ \,(1\leqslant i\leqslant n_a)$ allows both binary and multiple values. 
A function $f\in\mathcal{F}\!: \mathcal{X}\mapsto\mathcal{Y}$ represents a hypothesis in a space of hypotheses $\mathcal{F}$, of which the prediction for one instance $\bm{x}$ is denoted by $f(\bm{x})$ or $\hat{y}$ for brevity.

\subsection{Model fairness assessment from a manifold perspective}
\label{sec:method,dist}
In this paper, we mainly discuss the fairness measure in scenarios for sensitive attributes with binary values, that is, $a_i\in\mathcal{A}=\{0,1\}$ and $a_i=1$ represents the majority or privileged group for $i\in[n_a]$. 
Note that the observations could be extended to sensitive attributes with multiple values as well. 
In the case of instances with only one sensitive attribute, that is, $n_a=1$ and $\insx=(\xneg,a_1)$, the original dataset $S$ can be divided into a few disjoint sets according to the value of the sensitive attribute $a_1$, that is, 
$S_j=\{(\insx,y)\in S\mid a_1=j\},\,\forall j\in\mathcal{A}$. 
Because $\mathcal{A}=\{0,1\}$ is considered in this paper, we could get $S_0$ and $S_1$, while the division also works for sensitive attributes with multiple values, where $\mathcal{A}=\{0,1,2,...\}$ is a finite set. 
In other words, as $a_i=1$ represents the majority or privileged group, two subsets $S_1$ and $S_0$ denote the privileged group and the unprivileged group(s), respectively. 

Inspired by the principle of individual fairness---similar treatment for similar individuals, \emph{if we view the instances as data points on manifolds, the manifold representing members from the unprivileged group(s) is supposed to be as close as possible to that representing members from the privileged group} (\textbf{motivation}). 
A ``manifold'' refers to a topological space that locally resembles Euclidean space in mathematics. 
Then given a specific distance metric on instances $\newdist(\cdot,\cdot)$\footnote{Here we use the standard Euclidean metric. In fact, any two metrics $\mathbf{d}_1,\mathbf{d}_2$ derived from norms on the Euclidean space $\mathbb{R}^d$ are equivalent in the sense that there are positive constants $c_1,c_2$ such that $c_1\mathbf{d}_1(x,y)\leqslant \mathbf{d}_2(x,y)\leqslant c_2\mathbf{d}_1(x,y)$ for all $x,y\in \mathbb{R}^d$.}, we can define the distance between sets (that is, $S_0$ and $S_1$ here) by
\begin{equation}
\topequation
\begin{split}
\newDist(S_0,S_1) \defineq \max\big\{
    & \max_{(\insx,y)\in S_0}
    \overbrace{
    \min_{(\insx',y')\in S_1}
    \newdist\big( (\xneg,y), ({\xneg}',y') \big)
    }^{\text{for $(\insx,y)$ in $S_0$, to find the nearest point in $\bar{S}_1$}}
    \,,\\
    & \max_{(\insx',y')\in S_1}
    \underbrace{
    \min_{(\insx,y)\in S_0}
    \newdist\big( (\xneg,y), ({\xneg}',y') \big)
    }_{\text{for $({\insx}',y')$ in $S_1$, to find the nearest point in $S_0$}}
\big\} \,,\label{eq:1}
\end{split}%
\end{equation}%
and view it as an approximation of the distance between the manifold of unprivileged groups and that of the privileged group. Notice that the distance defined above satisfies the following basic properties:
\begin{enumerate}[label=(\arabic*)]
\item For any two data sets $S_0,S_1\in \mathcal{X}\times \mathcal{Y}$,\, $\newDist(S_0,S_1)=0$ if and only if $S_0$ equals $S_1$; and 
\item For any sets $S_0,S_1,$ and $S_2$, we have the triangle inequality 
$
    \newDist(S_0,S_2)\leqslant \newDist(S_0,S_1)+\newDist(S_1,S_2) 
    \,.\nonumber
$
\end{enumerate}
Analogously, for a trained classifier $f(\cdot)$, we can calculate 
\begin{equation}
\topequation
\begin{split}
\newFist(S_0,S_1)=\! \max\big\{ 
    & \max_{(\insx,y)\in S_0} \min_{(\insx',y')\in S_1} 
    \newdist\big( (\xneg,\hat{y}), ({\xneg}',\hat{y}') \big) \,,\\
    & \max_{(\insx',y')\in S_1} \min_{(\insx,y)\in S_0} 
    \newdist\big( (\xneg,\hat{y}), ({\xneg}',\hat{y}') \big)
\big\}.\label{eq:2}
\end{split}%
\end{equation}%
We remark that $\newDist(S_0,S_1)$ reflects the biases from the data, and $\newFist(S_0,S_1)$ reflects the biases from the algorithm. Then the following value could be used to reflect the fairness degree of this classifier, that is,
\begin{equation}
\topequation
    \newfist(f)= 
    \frac{\newFist(S_0,S_1)}{\newDist(S_0,S_1)}-1
    \,.\label{eq:3}
\end{equation}%

We name the fairness degree of one classifier by Eq.~\eqref{eq:3} as ``\emph{\ppsdisfull{} (}\ppsdisabbr\emph{)}''. 
\ppsdisabbr{} is designed to 
consider both individual- and group-fairness aspects simultaneously, 
and the reason why it can achieve this goal is that: i) \ppsdisabbr{} is built upon the concept of distance between sets, which includes distances between individuals as a basis; ii) the similarity between instances is indicated by the distance between individuals, while the disparity between two subgroups divided by sensitive attributes is indicated by the distance between sets; and iii) by taking them together, \ppsdisabbr{} is able to reflect fairness from both individual- and group-level aspects.

We further give a few remarks on $\newfist(f)$ defined in Eq.~\eqref{eq:3}:
\begin{itemize}
\item Notice a degenerate case that $\newdist((\xneg,y_1),(\xneg,y_2))=|y_1-y_2|$. We set $X_i=\{\xneg\mid (\insx,y)\in S_i\}$ and $Y_i=\{y\mid (\insx,y)\in S_i\}$ for $i\in \{0,1\}$. If $X\defineq X_0=X_1$, then $\newDist(S_0,S_1)\leqslant \max_{\xneg\in X}\big\{
|y-y'|\big\}.$ Similarly we have $\newDist_f(S_0,S_1)\leqslant \max_{\xneg\in X}\big\{
|f(\insx)-f(\insx')|\big\}.$ In particular, if $\newDist(S_0,S_1)>0$ (resp. $\newDist_f(S_0,S_1)>0$), then there exist two data points with exactly the same non-sensitive attributes, while their sensitive attributes and corresponding labels (predictions) are different.
\item $\newfist(f)$ measures the bias of the classifier concerning the biases from the data. In particular, if $\newfist(f)=0$, then $\newFist(S_0,S_1)=\newDist(S_0,S_1)$, indicating that the classifier does not introduce additional biases beyond those inherent in the data. If $\newfist(f)>0$, then $\newFist(S_0,S_1)>\newDist(S_0,S_1)$, suggesting that the algorithm introduces biases beyond those initially present in the dataset. And as $\newfist(f)$ grows, the relative bias of the classifier with respect to the biases from the dataset increases. Lastly, if $\newfist(f)<0$, then the classifier somehow reduces the bias arising from the original data.
\item Assume that $\newDist(S_0,S_1)=0$, which means that the original dataset is perfectly unbiased. If $\newFist(S_0,S_1)=0$ as well, then we write $\newfist(f)=\frac{0}{0}-1=0$, indicating that the classifier is also unbiased. On the other hand, if $\newFist(S_0,S_1)>0$, then $\newfist(f)=+\infty$, indicating that the classifier introduces biases by providing different outputs for identical data with differing sensitive attributes.
\item As we can see in the definition, 
$\newfist(f)$ is achieved by $\newdist((\xneg,y),({\xneg}',y'))$ for some $(\insx,y)\in S_0$ and $(\insx',y')\in S_1$, and this implies that $\newfist(f)$ is highly sensitive to the instances used in its computation.
\end{itemize}

\ppsdisabbr{} is interpretable in relation to data bias, because it provides a signed, relative notion of model-induced bias with a direct interpretation against the original data manifold discrepancy. 
Note that, unlike common fairness measures, \ppsdisabbr{} rather isolates model-induced bias relative to the bias already present in the data and assesses the algorithmic behaviour. This distinction is important in practice, as it enables practitioners to assess whether a model exacerbates existing disparities beyond what is inherited from the data. The practitioners are therefore advised to accompany other fairness measures, because a model cannot be claimed to be fair only due to $\newfist(f)=0$. 
Also note that $n_0$ and $n_1$ are the number of instances in $S_0$ and $S_1$, respectively. Then the computational complexity of directly calculating Equations~\eqref{eq:1} or \eqref{eq:2} would be $\mathcal{O}(n_0n_1)$, that is, $\mathcal{O}(n_0(n-n_0))$, which is expensive and less practical to be applied on large datasets.

\subsection{A quick approximation of distances between sets for Euclidean spaces}
Given the high computational complexity of directly calculating Equations \eqref{eq:1} and \eqref{eq:2}, that is, $\comp(n^2)$, it would be necessary to speed up the calculation if we intend to use them to measure the discriminative level of classifiers in practice.

Notice that the core operation in Equations~\eqref{eq:1}  and \eqref{eq:2} is evaluating the distance between data points inside $\mathcal{X}\times\mathcal{Y}$. To reduce the number of distance evaluation operations involved in Equations~\eqref{eq:1} and \eqref{eq:2}, we observe that the distance between similar data points tends to be closer than others after projecting them onto a general one-dimensional linear subspace. In fact, let $g: \mathcal{X}\times\mathcal{Y}\to \mathbb{R}$ be a projection, then 
\begin{equation}
    \topequation 
    \label{eq: inequality of projection}
    |g(\insx,y)-g(\insx',y')|\leqslant \mathbf{d}((\xneg,y),({\xneg}',y')) \,.
\end{equation}
To be concrete, one possible candidate for the projection $g$ could be
\begin{subequations}
\topequation
\begin{align}
    g(\insx,y;\vecw)=
    & g(\xneg,\xpos,y;\vecw)
    = [y,x_1,...,x_{n_x}]^\mathsf{T}\vecw 
    \,,\label{eq:4a}\\
    g(\insx,f;\vecw)= 
    & g(\xneg,\xpos,f(\insx);\vecw)
    = [\hat{y},x_1,...,x_{n_x}]^\mathsf{T}\vecw 
    \,,\label{eq:4b}
\end{align}%
\label{eq:4}%
\end{subequations}%
where $\vecw=[w_0,w_1,...,w_{n_x}]^\mathsf{T}$ is a random vector that meets $w_i\in[-1,1]$ for any $i\in\{0,1,...,n_x\}$ and $\sum_{i=0}^{n_x}|w_i|=1$.

Now, we choose a random projection $g: \mathcal{X}\times \mathcal{Y}\to \mathbb{R}$, then we sort all the projected data points on $\mathbb{R}$. According to Eq.~\eqref{eq: inequality of projection}, \emph{it is likely that for the instance $(\insx,y)$ in $S_0$, the desired instance $\argmin_{(\insx',y')\in S_1} \newdist((\xneg,y),({\xneg}',y'))$ would be someone near it after the projection, and vice versa.} Thus, by using the projections in Eq.~\eqref{eq:4a}, we could accelerate this process in Eq.~\eqref{eq:1} by checking several adjacent instances rather than traversing the whole dataset. Analogously, we could accelerate the process in Eq.~\eqref{eq:2} through the projection in Eq.~\eqref{eq:4b} as well. 
For simplification, we record $\{y_i\}_{i=1}^n$ and $\{\hat{y}_i\}_{i=1}^n$ as one denotation, that is, 
$\newy$ could be the true label $y$ or the prediction by the classifier $f(\cdot)$ or $\hat{y}$, 
then we can rewrite Equations~\eqref{eq:1}, \eqref{eq:2}, and \eqref{eq:4} as
\begin{small}%
\topequation
\begin{align}
    \newDist_{\cdot}(S_0,S_1)=& \max\{ 
     \max_{(\insx,y)\in S_0} \min_{(\insx',y')\in S_1} 
     \newdist\big( (\xneg,\newy),({\xneg}',\newy') \big) 
    \,,\nonumber\\
    &\hspace{2.1em} \max_{(\insx',y')\in S_1} \min_{(\insx,y)\in S_0} \newdist\big( ({\xneg},\newy),({\xneg}',\newy') \big)
    \} \,,\label{eq:5}\\
    g(\insx,\newy;\vecw)=&\ g(\xneg,\xpos, \newy;\vecw)
    =[\newy, x_1,...,x_{n_x}]^\mathsf{T} \vecw
    \,,\label{eq:6}
\end{align}%
\end{small}%
respectively. 
Note that $\newDist_{\cdot}(\cdot,\cdot)$ is symmetrical.

\begin{algorithm}[tb]
\small\caption{\small
\ppsalgfull, \aka{} 
\texttt{\ppsalgabbr}$(\{(\xneg_i,\xpos_i)\}_{i=1}^n, \{\newy_i\}_{i=1}^n; m_1,m_2)$
}\label{alg:approx}
\begin{algorithmic}[1]
    \REQUIRE Dataset $S\!=\!\{(\insx_i,y_i)\}_{i=1}^n \!=\!\{(\xneg_i,\xpos_i,y_i)\}_{i=1}^n$, prediction of $S$ by the classifier $f(\cdot)$ that has been trained, that is, $\{\hat{y}_i\}_{i=1}^n$, and two hyperparameters $m_1$ and $m_2$ as the designated numbers for repetition and comparison respectively
    \ENSURE Approximation of distance $\newDist_\cdot(S_0,S_1)$ in Eq.~\eqref{eq:5}
    \FOR{$j$ from $1$ to $m_1$}
    \STATE Take a random vector $\vecw$ from the space $\mathcal{W}=\{ \vecw=[w_0,w_1,...,w_{n_x}]^\mathsf{T} \mid \sum_{i=0}^{n_x}|w_i|=1 \}\subseteq[-1,1]^{1+n_x}$
    \STATE $d_\text{max}^{j}=$ \texttt{\ppssubabbr}$(\{(\xneg_i,\xpos_i)\}_{i=1}^n, \{\newy_i\}_{i=1}^n, \vecw;m_2)$
    \ENDFOR
    \RETURN $\min\{d_\text{max}^j \mid j\in[m_1]\}$ \label{alg:1,ln5} 
\end{algorithmic}
\end{algorithm}
\begin{algorithm}\small
\caption{\small%
\ppssubfull, \aka{} 
\texttt{\ppssubabbr}$(\{(\xneg_i,\xpos_i)\}_{i=1}^n, \{\newy_i\}_{i=1}^n, \vecw;m_2)$
}\label{alg:accele}
\begin{algorithmic}[1]
    \REQUIRE Data points $\{(\xneg_i,\xpos_i)\}_{i=1}^n$, its corresponding value $\{\newy_i\}_{i=1}^n$, where $\newy_i$ could be its true label $y_i$ or prediction $\hat{y}_i$ by the classifier $f(\cdot)$, a random vector $\vecw$ for projection, and a hyperparameter $m_2$ as the designated number for comparison
    \ENSURE Approximation of distance $\newDist_\cdot(S_0,S_1)$ in Eq.~\eqref{eq:5}
    \STATE Project data points onto a one-dimensional space based on Eq.~\eqref{eq:6}, in order to obtain $\{g(\insx_i,\newy_i;\vecw)\}_{i=1}^n$
    \label{line:acc,1}
    \STATE Sort original data points based on $\{g(\insx_i,\newy_i;\vecw)\}_{i=1}^n$ as their corresponding values, in ascending order
    \label{alg:acc,l2}
    \FOR{$i$ from $1$ to $n$} \label{line:acc,3}
    \STATE Set the anchor data point $(\insx_i,\newy_i)$ in this round \label{line:acc,4}
    \LineComment{If $\xpos_i\!=\!j \in\!\{0,1\}$, in order to approximate $\min_{(\insx',y')\in S\setminus S_j} \newdist( (\xneg_i,\newy_i),(\xneg',\newy') )$}
    \STATE Compute the distances $\newdist( (\xneg_i,\newy_i) ,\cdot )$ for at most $m_2$ nearby data points that meets $\xpos\neq\xpos_i$ and $g(\insx,\newy;\vecw) \leqslant g(\insx_i,\newy_i;\vecw)$\label{line:acc,5}
    \STATE Find the minimum among them, recorded as $d_\text{min}^{s}$
    \STATE Compute the distances $\newdist( (\xneg_i,\newy_i) ,\cdot )$ for at most $m_2$ nearby data points that meets $\xpos\neq\xpos_i$ and $g(\insx,\newy;\vecw) \geqslant g(\insx_i,\newy_i;\vecw)$\label{line:acc,7}
    \STATE Find the minimum among them, recorded as $d_\text{min}^r$
    \STATE $d_\text{min}^{(i)}= \min\{d_\text{min}^s, d_\text{min}^r\}$ \label{line:acc,9}
    \ENDFOR \label{line:acc,10}
    \RETURN $\max\{d_\text{min}^{(i)} \mid i\in[n]\}$
\end{algorithmic}
\end{algorithm}

Then we could propose an approximation algorithm to estimate the distance between sets, named as ``\emph{\ppsalgfull{} (}\ppsalgabbr\emph{)}'', shown in Algorithm~\ref{alg:approx}. 
Note that there exists a sub-route within \ppsalgabbr{} to obtain an approximated distance between sets, which is named as ``\emph{\ppssubfull{} (}\ppssubabbr\emph{)}'' and shown in Algorithm~\ref{alg:accele}. 
The sub-route will be repeatedly executed several times with independently sampled random projection vectors, and this design is intended to reduce the chance of missing true nearest neighbours under an unfavourable projection direction. 
As the time complexity of sorting in line~\ref{alg:acc,l2} of Algorithm~\ref{alg:accele} could reach $\mathcal{O}(n\log n)$, we could get the computational complexity of Algorithm~\ref{alg:accele} as follows: 
i) The complexity of line~\ref{line:acc,1} is $\comp(n)$; and ii) The complexity from line~\ref{line:acc,4} to line~\ref{line:acc,9} is $\comp(2m_2+1)$. 
Thus the overall time complexity of Algorithm~\ref{alg:accele} would be $\comp(n(1+\log n+2m_2))$ and that of Algorithm~\ref{alg:approx} be $\comp(m_1n(\log n+m_2))$. 
As both $m_1$ and $m_2$ are the designated constants, the time complexity of computing the distance is reduced to $\comp(n\log n)$, which is more desirable than the $\comp(n_0n_1)$ complexity for the direct computation in Section~\ref{sec:method,dist}, which is upper bounded by $\comp(n^2)$.\looseness=-1

It is worth noting that in line \ref{alg:1,ln5} of Algorithm~\ref{alg:approx}, we use the minimum instead of their average value. 
The reason is that in each projection, the exact distance for one instance would not be larger than the calculated distance for it via \ppssubabbr{}; and the same observation holds for all of the projections in \ppsalgabbr{}. 
Thus, the calculated distance via \ppsalgabbr{} is always no less than the exact distance, and the minimal operator should be taken finally after multiple projections. 
It also needs to be restated that neither \ppsalgabbr{} nor \ppssubabbr{} calculate $\newDist_\cdot(S_0,S_1)$ in the same way as its definition in Eq.~\eqref{eq:5}, and they are supposed to provide its approximated value. The idea behind them is that: i) the distance between similar data points tends to be closer than others after them being projected onto a general one-dimensional linear subspace, with theoretical foundations presented in Section~\ref{subsec:method,analysis}; and ii) thus, by using the projections, we can accelerate the process in Eq.~\eqref{eq:5} by checking several adjacent instances rather than traversing the whole dataset.\looseness=-1

\subsection{Algorithmic effectiveness analysis of \ppsalgabbr}
\label{subsec:method,analysis}%

As \ppsalgabbr{} in Algorithm~\ref{alg:approx} is devised to facilitate the approximation of direct calculation of the distance between sets, in this subsection, we detail more about its algorithmic effectiveness under some conditions.

Before delving into the main result (Proposition~\ref{density}), we first introduce an important Lemma that confirms the observation that `the distance between similar data points tends to be closer than others after projecting them onto a general one-dimensional linear subspace'. 
Note that projecting to a certain fixed line $L$ in the Euclidean space is like choosing a fixed unit vector $\vecw$ (in the direction of $L$) and consider the inner product $\langle \vecw,\bm{v} \rangle$ for each $\bm{v}$, the number $\langle \vecw,\bm{v} \rangle$ means the length of the projected vector of $\bm{v}$ on $L$, for simplicity we denote it as $\mathrm{Proj}_L(\bm{v})$. We notice that $\langle \vecw,\bm{v} \rangle$ has a sign, which stands for the direction of the projected vector on the line $L$. The absolute value of $\langle \vecw,\bm{v} \rangle$ means the real distance.

\begin{lemma}\label{projection 1}
Let $\bm{v}_1$ (resp. $\bm{v}_2$) be a vector in the $n$-dimensional Euclidean space $\mathbb{R}^n$ with length $r_1$ (resp. $r_2$) such that $r_1\leqslant r_2$. 
Let $\vecw\in \mathbb{R}^n$ be a unit vector. 
We define $\mathbb{P}(\bm{v}_1,\bm{v}_2)$ as the probability that $|\langle \vecw,\bm{v}_1\rangle|\geqslant |\langle \vecw,\bm{v}_2\rangle|$. Then,
\begin{equation}\label{eq: ineq conclusion}
\topequation%
    \frac{\mathrm{sin}~\phi}{\pi} \cdot\frac{r_1}{r_2}
    \leqslant \mathbb{P}(\bm{v}_1,\bm{v}_2)\leqslant 
    \bigg(1+\frac{r_1^2}{r_2^2}\bigg)^{-\sfrac{1}{2}}
    \cdot\frac{r_1}{r_2} \,,
\end{equation} here $\phi$ represents the angle between $\bm{v}_1$ and $\bm{v}_2$.
\end{lemma}
\begin{proof}
Notice that $|\langle \vecw,\bm{v}_1\rangle|\geqslant |\langle \vecw,\bm{v}_2\rangle|$ is equivalent to 
\begin{equation}
\topequation%
\label{equ: condition for projection}
    \langle \bm{v}_2-\bm{v}_1,\vecw\rangle \langle \bm{v}_1+\bm{v}_2,\vecw\rangle\leqslant 0 \,.
\end{equation}
If $\vecw$ satisfies Eq.~\eqref{equ: condition for projection}, then it lies between two hyperplanes that are perpendicular to $\bm{v}_1-\bm{v}_2$ and $\bm{v}_1+\bm{v}_2$ respectively. Denote by $\theta$ the angle between these two hyperplanes (which is equal to the acute angle between $\bm{v}_2-\bm{v}_1$ and $\bm{v}_1+\bm{v}_2$), and  $0\leqslant \theta\leqslant \frac{\pi}{2}$, then $\mathbb{P}(\bm{v}_1,\bm{v}_2)=\frac{\theta}{\pi}$.\footnote{%
Since the event only depends on the two inner products with $\bm v_2-\bm v_1$ and $\bm v_1+\bm v_2$, by rotational invariance, it suffices to consider the two-dimensional plane spanned by these two vectors. In this plane, the set of directions satisfying the above inequality has total angle $2\theta$ out of $2\pi$.
} Moreover, 
\begin{equation}
\topequation
\label{eq: ineq cor1 1}
    \mathrm{sin}^2~\theta=1-\mathrm{cos}^2~\theta=\frac{4\lVert \bm{v}_1\rVert^2\lVert \bm{v}_2\rVert^2-4\langle \bm{v}_1,\bm{v}_2\rangle^2}{(\lVert \bm{v}_1\rVert^2+\lVert \bm{v}_2\rVert^2)^2-4\langle \bm{v}_1,\bm{v}_2 \rangle^2} \,.
\end{equation}
Here $\lVert \bm{v}_i\rVert^2=\langle \bm{v}_i,\bm{v}_i\rangle=r_i^2$ ($i=1,2$) is the square length of the vector. Recall that $\langle \bm{v}_1,\bm{v}_2\rangle=\lVert \bm{v}_1\rVert \lVert \bm{v}_2\rVert \mathrm{cos}~\phi$. By Eq.~\eqref{eq: ineq cor1 1}, we have\footnote{%
Note that the left hand side is obtained based on $(r_1^2+r_2^2)^2-4r_1^2r_2^2\cos^2\phi \leqslant (r_1^2+r_2^2)^2\leqslant 4r_2^4$, and the right hand side is obtained based on $(r_1^2+r_2^2)^2\sin^2\phi \leqslant (r_1^2+r_2^2)^2-4r_1^2r_2^2 \cos^2\phi$, that is, $0\leqslant ((r_1^2+r_2^2)^2-4r_1^2r_2^2)\cos^2\phi =(r_2^2-r_1^2)^2\cos^2\phi$.
}%
\begin{equation}
\topequation
\label{eq: ineq cor1 2}
    \frac{\lVert \bm{v}_1\rVert^2}{\lVert \bm{v}_2\rVert^2}\mathrm{sin}^2~\phi\leqslant \mathrm{sin}^2~\theta\leqslant\frac{4\lVert \bm{v}_1\rVert^2\lVert \bm{v}_2\rVert^2}{(\lVert \bm{v}_1\rVert^2+\lVert \bm{v}_2\rVert^2)^2} \,.
\end{equation}
Combining Eq.~\eqref{eq: ineq cor1 2} with the fact that $\frac{2}{\pi}\theta\leqslant\mathrm{sin}~\theta\leqslant \theta$, we conclude that the probability $\mathbb{P}(\bm{v}_1,\bm{v}_2)=\frac{\theta}{\pi}$ satisfies the desired inequalities. 
Note that by Eq.~\eqref{eq: ineq conclusion}, when the ratio $r_1/r_2$ goes to zero, the probability $\mathbb{P}(\bm{v}_1,\bm{v}_2)$ goes to zero; Also, when $r_1=r_2$ and $\bm{v}_1 \neq\pm \bm{v}_2$, it is easy to calculate that $\mathbb{P}(\bm{v}_1,\bm{v}_2)=\frac{1}{2}$, while the cases $\bm v_1=\pm\bm v_2$ are degenerate.
\end{proof}

\begin{figure}[tb]%
\centering
\includegraphics[width=5.3cm]{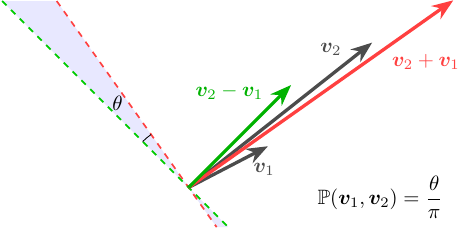}%
\vspace{-2mm}
\caption{Illustration of Lemma~\ref{projection 1}. 
The dashed lines represent the hyperplanes that are perpendicular to $(\bm{v}_1-\bm{v}_2)$ and $(\bm{v}_1+\bm{v}_2)$ respectively, and $\theta$ denotes the acute angle between the hyperplanes.}\label{lem 1}
\end{figure}

To understand Lemma~\ref{projection 1}, suppose that we have two vectors $\bm{v}_i$ with length $r_i$. We can choose a random line $L_\text{random}$ of the Euclidean space, and we consider the projection to $L_\text{random}$. We know that for some $L_\text{random}$, it may happen that after projection, the length of $\bm{v}_1$ is larger than the length of $\bm{v}_2$, and for the rest $L_\text{random}$, the opposite relation holds. Now we suppose that $r_1<r_2$ (actually in application, we may assume that $r_1 \ll r_2$ in most cases), and we randomly choose a line $L$ (to do the projection), what is the probability that, after projecting onto $L$, we still have length $\mathrm{Proj}_L(\bm{v}_1) < \mathrm{Proj}_L(\bm{v}_2)$? 
\emph{The lemma says that the probability of reversing the original size relationship of the lengths of vectors is at most of order $\frac{r_1}{r_2}$, and hence the probability of preserving the original size relationship is at least $1-\mathcal{O}(\frac{r_1}{r_2})$.} 
In particular, if $\frac{r_1}{r_2}$ goes to zero (the length of $\bm{v}_2$ is much larger than the length of $\bm{v}_1$), then the probability that $\mathrm{Proj}_L(\bm{v}_1) < \mathrm{Proj}_L(\bm{v}_2)$ goes to one. Now we take three data points $\bm{p}_1,\bm{p}_2,\bm{p}_3$, and we may consider the distance vector $\bm{v}_1=\bm{p}_2-\bm{p}_1$ and $\bm{v}_2=\bm{p}_3-\bm{p}_1$, the length of these vectors corresponds to the distance of corresponding data points. We may apply the previous lemma to $\bm{v}_1$ and $\bm{v}_2$ to estimate the probability that after taking a projection to a random line, the projection keeps the size relationship of the distances between the given points (if the distance between $\bm{p}_1$ and $\bm{p}_2$ is less than the distance of $\bm{p}_1$ and $\bm{p}_3$, then with the probability estimated above and after projecting to the line, the three projected points $\bm{q}_1,\bm{q}_2, \bm{q}_3$ satisfy that the distance of $\bm{q}_1$ and $\bm{q}_2$ is less than the distance of $\bm{q}_1$ and $\bm{q}_3$). This explains why ``the distance between similar data points tends to be closer than others after projecting them onto a general one-dimensional linear subspace'' in our observation. 
Furthermore, to compute the distance between two sets of points usually involves a lot of computation of distances of points and comparing the size relationship of lengths. To avoid the huge computation, we can project the data points to random lines and sort the projected points on this line. As the previous lemma suggests that, for a fixed comparison of two distances from the same anchor point, a random projection is unlikely to reverse the order when the shorter distance is much smaller than the longer one, and since such distance comparisons are likely to be preserved under a random projection in the above sense, we mainly need to evaluate the distances of data points that are nearby after projection, which will save a huge amount of computation.

\begin{proposition}\label{density}
Let $S= \{(\bm{x}_i,y_i)\}_{i=1}^n\subset \mathcal{X}\times \mathcal{Y}$ be a $(k+1)$-dimensional (instances have $k+1$ features) dataset, evenly distributed dataset with a size of $n$ that is a random draw of the training instances. For any two subsets of $S$ with distance $d$ (ref. Eq.~\eqref{eq:5}), suppose further that the scaled density    \begin{equation}\label{eq: scaled density}
\topequation%
\sup_r~\bigg\{
\inf_{\mathbf{B} \text{ ball of radius $r$}}\frac{1}{\mathrm{Vol}(\mathbf{B})}\#(\mathbf{B}\cap S)
\bigg\}= \frac{\mu}{\mathrm{Vol}(\mathbf{B}(d))}\,,
\end{equation}
for some positive real number $\mu$ (here $\#$ denotes the number of points of a finite set and $\mathbf{B}(d)$ denotes a ball of radius $d$). Then, with probability at least \begin{equation}\label{eq: probability}
\topequation%
1-\bigg(\frac{\pi\mu}{m_2\mathrm{Vol}(\mathbf{B}(1))} 
\Big(\Big(1+\frac{n}{\mu}\Big)^{\frac{1}{k+1}}-\alpha\Big)\bigg)^{m_1}\,,
\end{equation}
\ppsalgabbr{} could reach an approximate solution that is at most $\alpha$ times of the distance between these two subsets.
\end{proposition}

\begin{proof}
Let $S_0$ and $S_1$ be two sub-datasets of $S$. We fix the instance $\bm{v}_0\in S_0$ such that $d\defineq \mathbf{D}(S_0,S_1)=\mathbf{d}(\bm{v}_0,\bm{v_1})$ for some $\bm{v_1}\in S_1$. 
For simplicity, we may set $\bm{v}_0$ as the origin. The probability that an instance $\bm{v}\in S_1$ has a shorter length than $\bm{v}_1$ after projection to a line (see Eq.~\eqref{eq:4}) is denoted as $\mathbb{P}(\bm{v}_1,\bm{v})$. By assumption, we only need to consider those instances whose length is greater than $\alpha d$ (outside the ball $\mathbf{B}(\alpha d)$ centered at the origin), that is, $\{\bm{v} \mid \bm{v}\notin \mathbf{B}(\alpha d) \}$. 
Then the probability that \ppssubabbr{} fails to include such a candidate is bounded above by $\frac{1}{m_2} \sum_{\bm{v}\notin \mathbf{B}(\alpha d)} \probP(\bm{v}_1, \bm{v})$; the probability that we could not find an approximate solution by \ppsalgabbr{} is 
$( \tfrac{1}{m_2} \sum_{\bm{v}\notin \mathbf{B}(\alpha d)} \probP(\bm{v}_1, \bm{v}) )^{m_1}
.$ %
Hence, the desired probability is bounded from below by
\begin{equation}\label{eq: probability 1}
\topequation%
    1-\bigg(\frac{1}{m_2} \sum_{\bm{v}\notin \mathbf{B}(\alpha d)} \mathbb{P}(\bm{v}_1,\bm{v})\bigg)^{m_1}\,.
\end{equation}

However, Eq.~\eqref{eq: probability 1} is based on the extreme assumption that all instances lie on the same two-dimensional plane. In our case, the instances are evenly distributed. Hence, we may adjust the probability by multiplying 
\begin{equation}
\topequation\nonumber
\frac{\volume\big(S^1(\frac{\lVert \bm{v}\rVert}{d})\big)}{\volume\big(S^k(\frac{\lVert \bm{v}\rVert}{d})\big)}=
\frac{\Gamma(\frac{k+1}{2})}{\pi^{\frac{k-1}{2}}} \cdot \Big( \frac{d}{\lVert \bm{v}\rVert}\Big)^{k-1}
\,,\end{equation}%
here $\Gamma(\cdot)$ denotes the Gamma function and $\volume(S^i(r))$ denotes the area of the $i$-th dimensional sphere of radius $r$. 
Note that $\volume(S^i(r))= \sfrac{ 2\pi^{\frac{i+1}{2}} \, r^{i} }{ \Gamma(\frac{i+1}{2}) }$.
Hence, by Lemma~\ref{projection 1}, the desired probability is lower bounded by 
\begin{equation}\label{eq: probability 11}
\topequation%
    1-\bigg(\frac{1}{m_2} \sum_{\bm{v}\notin \mathbf{B}(\alpha d)} \bigg(1+\frac{d^2}{\lVert \bm{v}\rVert^2}\bigg)^{-\frac{1}{2}} \frac{\Gamma(\frac{k+1}{2})}{\pi^{\frac{k-1}{2}}}
    \cdot\bigg(\frac{d}{\lVert \bm{v}\rVert}\bigg)^k\bigg)^{m_1}\,.
\end{equation}
Under our assumption, Eq.~\eqref{eq: probability 11} attains the lowest value when the data are evenly distributed inside a hollow ball $\mathbf{B}_0\setminus \mathbf{B}(d)$ centered at $\bm{v}_0$. The radius of $\mathbf{B}_0$, denoted as $r_0$, satisfies 
\begin{equation}\label{eq: radius}
\topequation%
    n-1=\mu\frac{\volball(\ball_0\setminus \ball(d))}{\volball(\ball(d))}=\mu\Big(\Big(\frac{r_0}{d}\Big)^{k+1}-1\Big)\,.
\end{equation}
Note that $\volball(\ball(d))$ denotes the volume of a ball of radius $d$, and, for a $i$-th dimensional ball, $\volball(\ball^i(d))\!= \sfrac{ \pi^\frac{i}{2} d^{i} }{ \Gamma(\frac{i}{2}+1) }$. 
Also note their relationship is $\volume(S^k(r))= \frac{d}{dr} \volball(\ball^{i+1}(r))$.

In this situation, we may write the summation part of Eq.~\eqref{eq: probability 11} as an integration. To be more specific, Eq.~\eqref{eq: probability 11} is lower bounded by
\begin{equation}\label{eq: probability 22}
\topequation%
    1-\bigg(\frac{1}{m_2} \int_{\alpha d}^{r_0} A(x)\mu \,\volume(S^{k}(x))dx \bigg)^{m_1}\,.
\end{equation}
where $A(x)=(1+\frac{d^2}{x^2})^{-\frac{1}{2}} \frac{\Gamma(\frac{k+1}{2})}{\pi^{(k-1)/2}} \cdot(\frac{d}{x})^k$. 
Moreover, Eq.~\eqref{eq: probability 22} can be simplified as 
\begin{equation}\label{eq: probability 33}
\topequation%
    1-\bigg(\frac{1}{m_2\volball(\ball(1))} \int_{\alpha d}^{r_0} \frac{\pi\mu}{d}\cdot \frac{x}{\sqrt{x^2+d^2}} dx \bigg)^{m_1}\,.
\end{equation}
Combining Eq.~\eqref{eq: radius} and \eqref{eq: probability 33}, we conclude that the desired probability is lower bounded by 
\begin{equation}\label{eq: probability 44}
\topequation%
    1-\bigg(\frac{\pi\mu}{m_2\mathrm{Vol}(\mathbf{B}(1))} 
    \bigg( \Big(\Big(1+\frac{n}{\mu}\Big)^{2/(k+1)}+1\Big)^{\frac{1}{2}}-(\alpha^2+1)^{\frac{1}{2}}\bigg)\bigg)^{m_1}\,.
\end{equation}
And the proposition follows from Eq.~\eqref{eq: probability 44}.
\end{proof}

The analysis in Proposition~\ref{density} relies on a uniform ball assumption, which is admittedly strong and not directly verifiable for real-world data. This assumption is introduced to enable a tractable analysis and to derive a non-trivial lower bound on the probability that \ppsalgabbr{} includes true nearest neighbours in the candidate set. Intuitively, the uniform ball condition ensures that each local neighbourhood has sufficient mass in all directions, allowing random projections to capture nearby points with bounded probability, while the scaled density parameter $\mu$ abstracts the worst-case local sparsity relative to the global scale.
When these assumptions are violated (for example, in the presence of dense clusters, anisotropic structures, or highly uneven data densities), the resulting bound may become loose or even vacuous, as the underlying geometric argument no longer applies. Importantly, this does not imply that \ppsalgabbr{} fails in practice. From an algorithmic perspective, violations of the uniform ball assumption primarily affect the tightness of the probabilistic bound, rather than the correctness of the approximation procedure itself, since \ppsalgabbr{} always produces an upper bound of the true set distance by construction. In fact, under clustered or non-uniform distributions, random projections often preserve local neighbour relations well, making the approximate candidate sets likely to contain true nearest neighbours with high empirical probability. Since clustered structures increase local point density, the probability that true nearest neighbours appear among nearby projected candidates is typically higher than the worst-case scenario captured by the uniform analysis. Moreover, our experiments on non-uniform and clustered datasets (see Sections~\ref{subsec:RQ2} to \ref{subsec:RQ3}) confirm that the approximation quality remains high beyond the scope of the current theoretical guarantees.
Extending the analysis to relax these assumptions using data-dependent geometric measures, such as intrinsic dimension or doubling dimension, is an interesting direction for future work.

Moreover, note that by Eq.~\eqref{eq: probability}, the efficiency of \ppsalgabbr{} decreases as the scaled density $\frac{\mu}{\volume(\ball(d))}$ of the original dataset increases. Meanwhile, when dealing with a large-scale dataset, the more insensitive attributes we have, the more efficient \ppsalgabbr{} is. In general, the efficiency of \ppsalgabbr{} depends on the shape of these two subsets of $S$. Roughly speaking, the more these two sets are separated from each other, the more efficient \ppsalgabbr{} is.

Now we discuss our choice of $m_1$ and $m_2$ according to Eq.~\eqref{eq: probability}. In fact, Eq.~\eqref{eq: probability} can be approximately written as $1-c\cdot n^{\frac{m_1}{k+1}}/m_2^{m_1}$. We can calculate the order of magnitude of $n^{\frac{m_1}{k+1}}/m_2^{m_1}$ by taking the logarithm: 
\begin{equation}\label{eq: magnitude}
\topequation%
-\lambda\defineq\mathrm{lg} \Big(n^{\frac{m_1}{k+1}}/m_2^{m_1}\Big)
=m_1\Big(\tfrac{\mathrm{lg}~n}{k+1}-\mathrm{lg}~m_2\Big)\,.
\end{equation}
The \ppsalgabbr{} could reach an approximate solution with probability at least $1-c\cdot 10^{-\lambda}$. In practice, we choose positive integers $m_2$ and $m_1$ such that $\lambda$ is reasonably large, ensuring that the algorithm will reach an approximate solution with high probability. 
That is to say, a fine-grained tuning of hyperparameters ($m_1$ or $m_2$) is not required, while they may grow moderately with the dataset size.

\section{Empirical Results}
In this section, we elaborate on our experiments to evaluate the effectiveness of the proposed \ppsdisabbr{} in Eq.~\eqref{eq:3} and \ppsalgabbr{} in Algorithm~\ref{alg:approx}. 
These experiments are conducted to explore the following research questions: 
\textbf{RQ1}. Compared with the prevailing baseline fairness measures, does the proposed \ppsdisabbr{} capture the discriminative degree of one classifier effectively, and can it capture the discrimination level from both individual and group fairness aspects? 
\textbf{RQ2}. Can \ppsalgabbr{} approximate the direct computation of distances in Eq.~\eqref{eq:5} precisely, and how efficient is \ppsalgabbr{} compared with the direct computation of distances? 
\textbf{RQ3}. Will the choice of hyperparameters (that is, $m_1$ and $m_2$ in \ppsalgabbr) affect the approximation results, and if the answer is yes, how?

\subsection{Experimental setups}
\label{subsec:setup}
In this subsection, we present the experimental setting, including datasets, evaluation metrics, baseline fairness measures, and implementation details.

\begin{table}[t]
\centering
\caption{Dataset statistics. 
The column `\#inst' represents the number of instances. 
}\label{tab:stats}
\vspace{-.7em}
\scalebox{.89}{%
\begin{tabular}{r |r|r|r |r|r}
    \toprule
    \multirow{2}{*}{\bf Datasets} & 
    \multirow{2}{*}{\bf \#inst} & 
    \multicolumn{2}{c|}{\bf \#feature} & 
    \multicolumn{2}{c}{\bf \#member in the privileged group} \\
    \cline{3-6}
    & & raw & binarized & 1st priv~~ & 2nd priv~~ \\
    \midrule
    ricci &   118 &  5 &   6 & 
    68 in race & --- \\
    credit&  1000 & 21 &  59 & 
    690 in sex & 851 in age \\
    income& 30162 & 14 &  99 & 
    25933 in race & 20380 in sex \\
    ppr   &  6167 & 11 & 402 & 
    4994 in sex & 2100 in race \\
    ppvr  &  4010 & 11 & 328 & 
    3173 in sex & 1452 in race \\
    \bottomrule
\end{tabular}
}
\end{table}

\paragraph{Datasets} 

Five public datasets were adopted in the experiments: 
Ricci,\footnote{\url{https://rdrr.io/cran/Stat2Data/man/Ricci.html}} 
Credit,\footnote{\url{https://archive.ics.uci.edu/dataset/144/statlog+german+credit+data}} 
Income (\aka{} Adult),\footnote{\url{https://archive.ics.uci.edu/dataset/2/adult}} 
COMPAS PPR, and COMPAS PPVR.\footnote{\url{https://github.com/propublica/compas-analysis/}, that is,\\Propublica-Recidivism and Propublica-Violent-Recidivism datasets} 
Each of them has two sensitive attributes except Ricci, with more statistical details provided in Table~\ref{tab:stats}, and is commonly used in the literature \citep[Table 2]{pessach2022review}. 
Among these datasets, Ricci was used in a court case dealing with racial discrimination; COMPAS PPR and PPVR were relevant to racial discrimination in criminal risk assessments.

\paragraph{Evaluation metrics} 
As data imbalance usually exists within unfair datasets, we consider several criteria to evaluate the prediction performance, including accuracy, precision, recall (\aka{} sensitivity), $\mathrm{f}_1$ score, and specificity. 
For efficiency metrics, we directly compare the time cost of different methods.\looseness=-1

\begin{figure*}
\centering%
\subfloat[]{\label{subfig:comp,a}
\includegraphics[height=28mm]{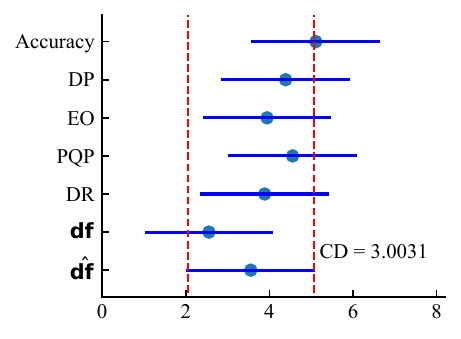}\vspace{2mm}}
\subfloat[]{\label{subfig:comp,b}
\includegraphics[height=31mm]{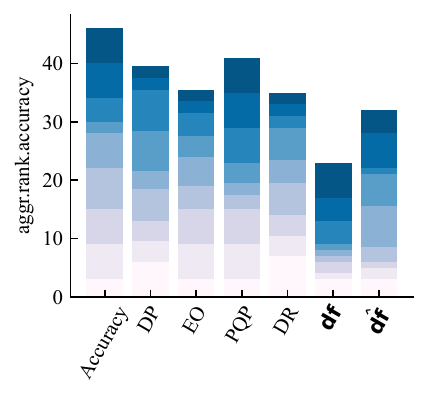}}
\subfloat[]{\label{subfig:comp,alt,a}
\includegraphics[height=28mm]{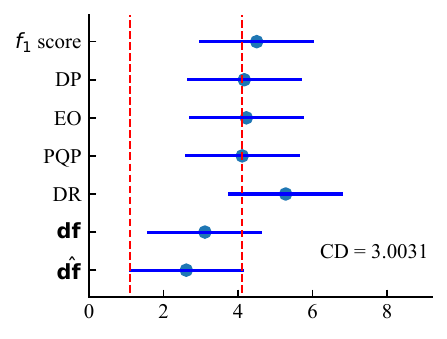}\vspace{2mm}}
\subfloat[]{\label{subfig:comp,alt,b}
\includegraphics[height=31mm]{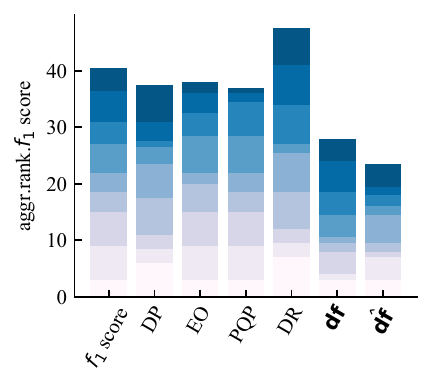}}
\\ \vspace{-3mm}%
\subfloat[]{\label{subfig:comp,d}
\includegraphics[height=31mm]{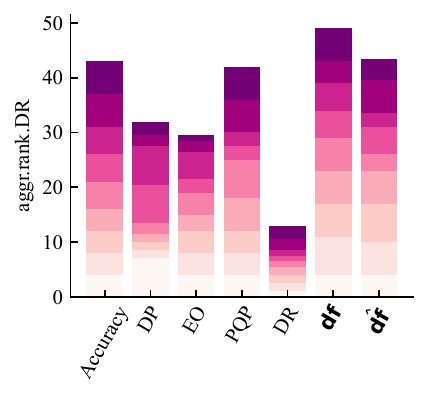}}
\hspace{1mm}
\subfloat[]{\label{subfig:fair,k0}
\includegraphics[width=33mm]{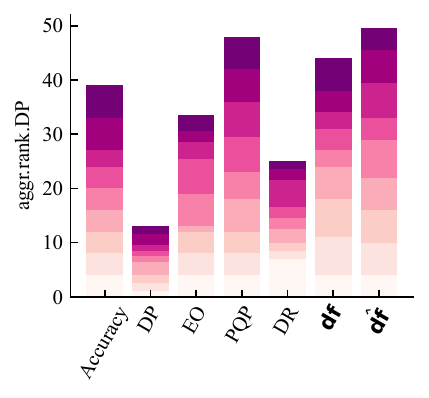}}
\hspace{1mm}
\subfloat[]{\label{subfig:fair,k1}
\includegraphics[width=33mm]{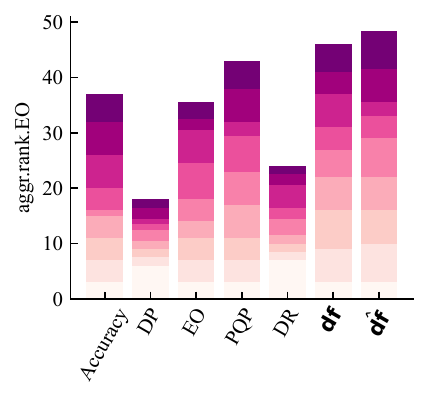}}
\hspace{1mm}
\subfloat[]{\label{subfig:fair,k2}
\includegraphics[width=33mm]{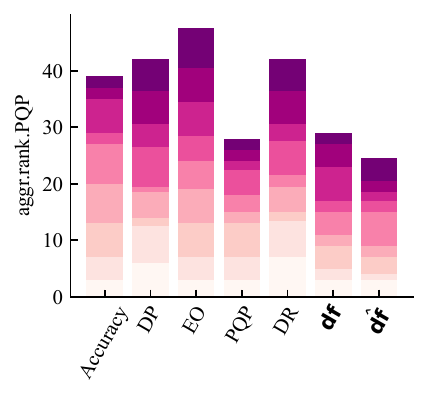}}
\\ \vspace{-4mm}%
\subfloat[]{\label{subfig:comp,alt,d}
\includegraphics[height=31mm]{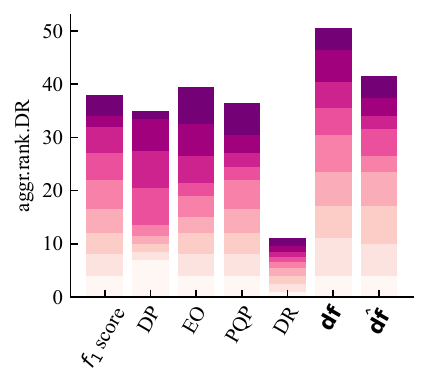}}
\hspace{1mm}
\subfloat[]{\label{subfig:falt,k0}
\includegraphics[width=33mm]{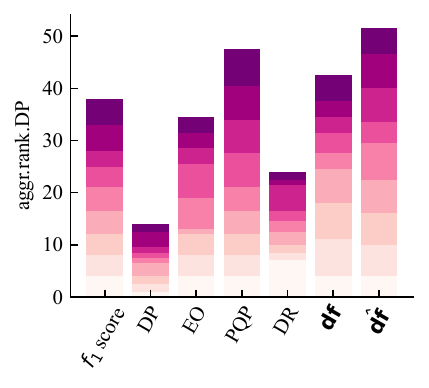}}
\hspace{1mm}
\subfloat[]{\label{subfig:falt,k1}
\includegraphics[width=33mm]{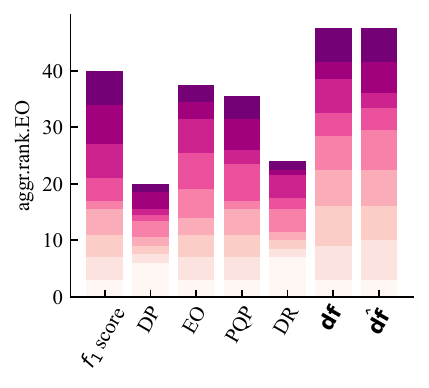}}
\hspace{1mm}
\subfloat[]{\label{subfig:falt,k2}
\includegraphics[width=33mm]{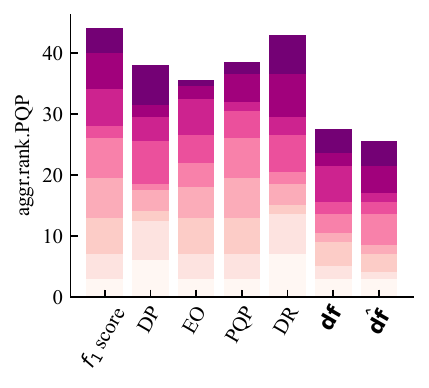}}
\vspace{-2mm}\caption{%
Comparison of baseline fairness measures and the proposed \ppsdisabbr{}. 
(a) Friedman test chart (non-overlapping means significant difference) on the test accuracy, which rejects the null hypothesis that `all fairness-relevant rules have the same evaluation performance' at the significance level of 5\%, 
and where CD means the critical difference of average rank difference, calculated by Nemenyi post-hoc test \cite{zhou2021machine}; (b) The aggregated rank of each fairness-relevant rule (the smaller the better) \cite{qian2015pareto} on the test accuracy; 
(c--d) Friedman test chart and the aggregated rank of each fairness-relevant rule on the $\mathrm{f}_1$ score, evaluated on test data. 
(e--h) The aggregated rank of each fairness-relevant rule on the DR and three group fairness measures, respectively; 
(i--l) The aggregated rank of each fairness-relevant rule on the DR, DP, EO, and PQP, respectively. 
Note that the ranking rule in (a--b) and (e--h) is $0.05\times\text{error rate}+0.95\times\text{fairness}$, and that in (c--d) and (i--l) is $0.1\times(1-\mathrm{f}_1\text{ score})+0.9\times\text{fairness}$. 
Also note that in all of them except (a) and (c), the smaller (in the $y$-axis), the better. 
}\label{fig:comp}
\end{figure*}

\begin{figure*}
\begin{minipage}{\textwidth}
\centering
\includegraphics[width=15cm]{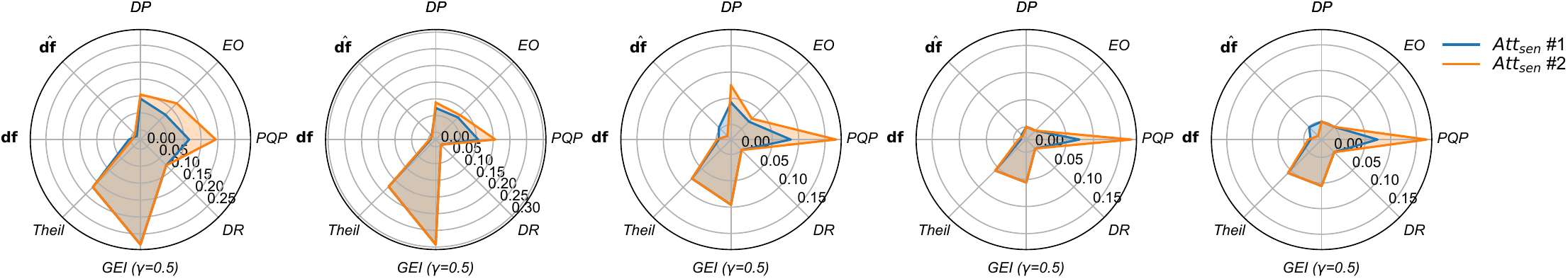}
\includegraphics[width=15cm]{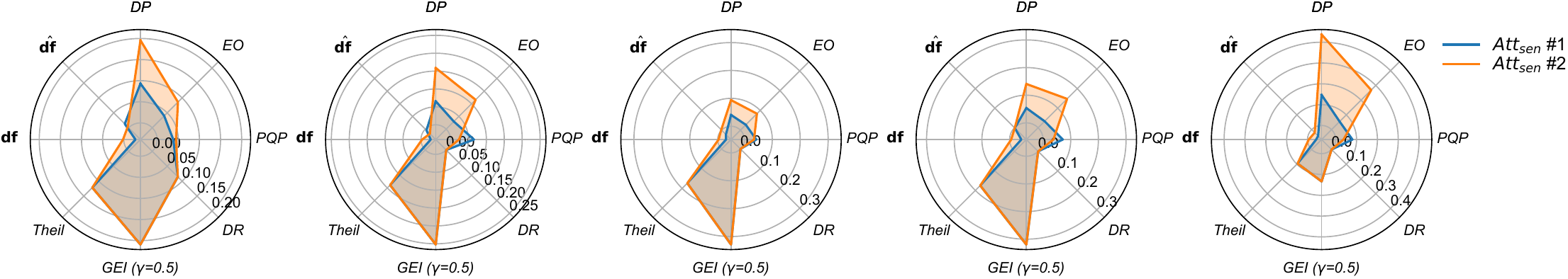}
\includegraphics[width=15cm]{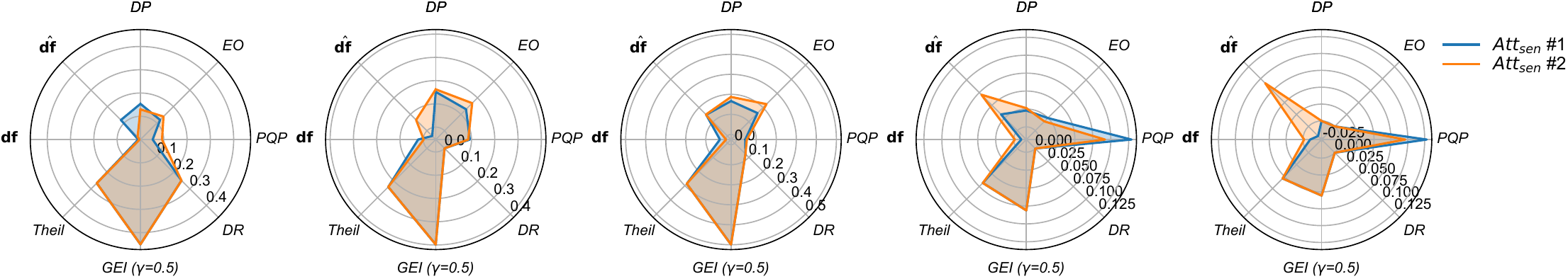}
\includegraphics[width=15cm]{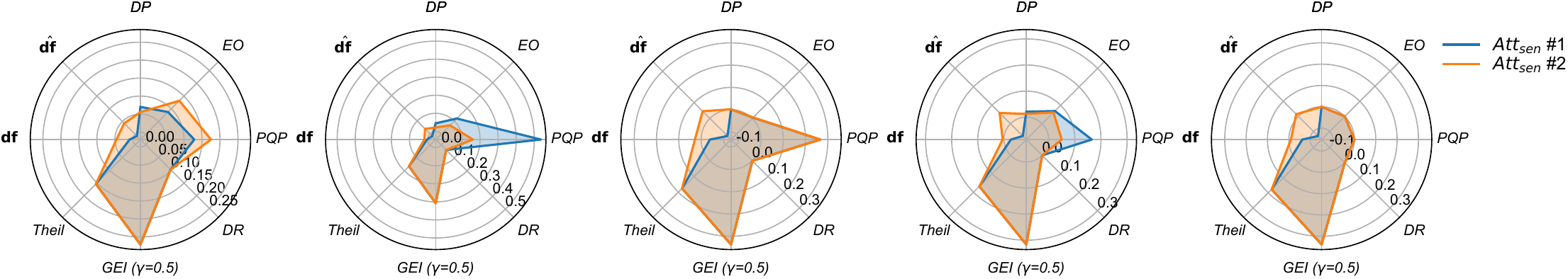}
\vspace{-4mm}\caption{
Comparison of baseline fairness measures and the proposed \ppsdisabbr{}, where $\text{Att}_\text{sen}$ \#i' indicates the unfairness degree of the corresponding algorithm regarding $a_i$, that is, the $i$-th sensitive attribute in the dataset. 
Four rows correspond to results on the Credit, Income, PPR, and PPVR datasets, respectively; Five columns correspond to results using bagging, AdaBoost, LightGBM, AdaFair (mitigating bias from the first sensitive attribute), and AdaFair (mitigating bias from the second sensitive attribute), respectively. 
Also note that $\gamma=0.5$ in the GEI computation.
}\label{fig:radar}
\end{minipage}
\begin{minipage}{\textwidth}
\vspace{-2mm}
\centering
\subfloat[]{\includegraphics[height=24mm]{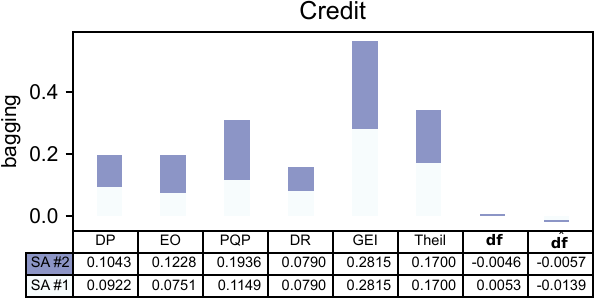}}
\hspace{1mm}
\subfloat[]{\includegraphics[height=24mm]{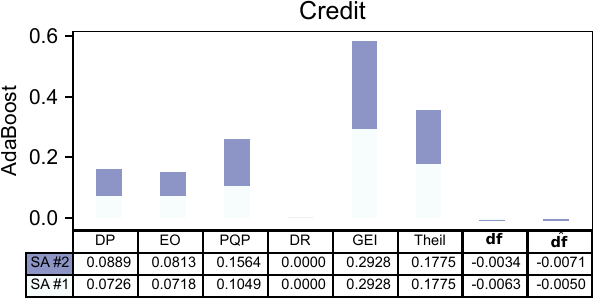}}
\hspace{1mm}
\subfloat[]{\includegraphics[height=24mm]{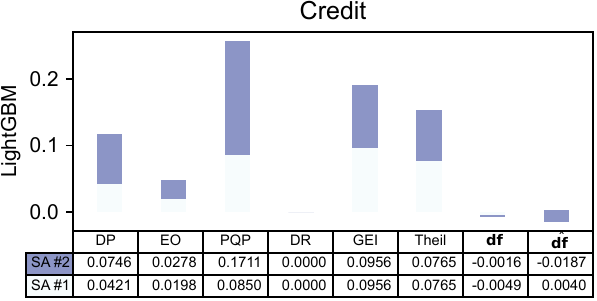}}
\\ \vspace{-4.5mm}
\subfloat[]{\label{subfig:tab4}%
\includegraphics[height=24mm]{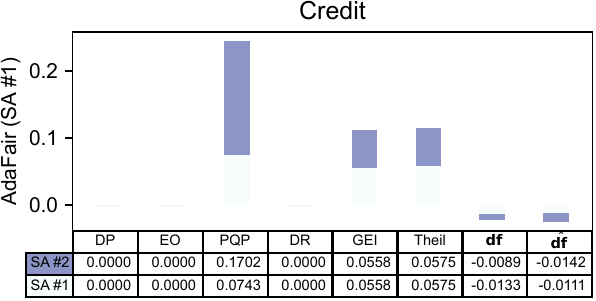}}
\hspace{1mm}
\subfloat[]{\label{subfig:tab5}%
\includegraphics[height=24mm]{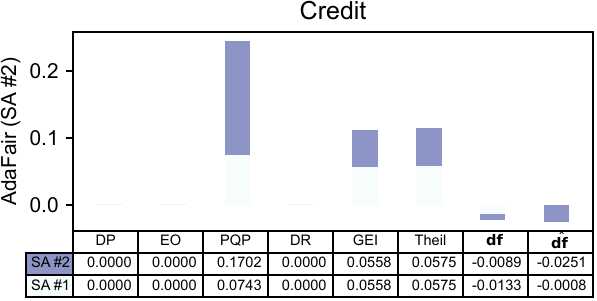}}
\vspace{-2mm}\caption{%
Another expression of sub-figures in the first row of Fig.~\ref{fig:radar}, where (a--e) corresponds to five columns (\ie{} learning algorithms) respectively.
}\label{fig:tabular}
\end{minipage}
\end{figure*}

\begin{figure}[t]
\begin{minipage}{\linewidth}
\centering%
\includegraphics[height=28mm]{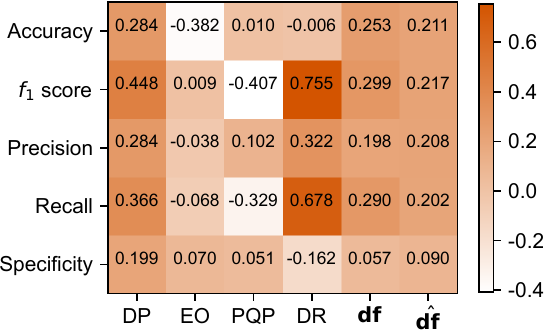}
\vspace{-4mm}%
\caption{Comparison of baseline fairness measures and the proposed \ppsdisabbr{}, evaluated on test data: 
Correlation heatmap between the evaluation metric and fairness. 
Note that $\newfist$ is calculated directly by Eq.~\eqref{eq:3}, and $\hat{\newfist}$ is calculated approximately with distances between sets computed using \ppsalgabbr{}. 
}\label{fig:delta,resbm}
\end{minipage}
\begin{minipage}{\linewidth}
\vspace{-1mm}
\centering 
\subfloat[]{\label{subfig:approx,a}
\includegraphics[width=3.45cm]{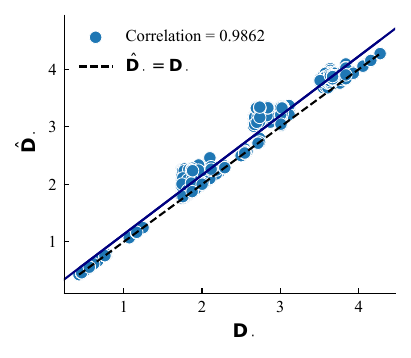}}
\hspace{3mm}
\subfloat[]{\label{subfig:approx,c}
\includegraphics[width=3.6cm]{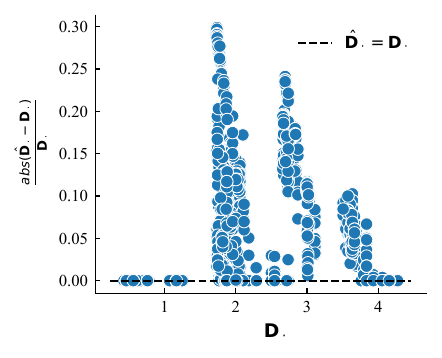}}
\\ \vspace{-4mm}
\subfloat[]{\label{subfig:approx,b}
\includegraphics[width=3.5cm]{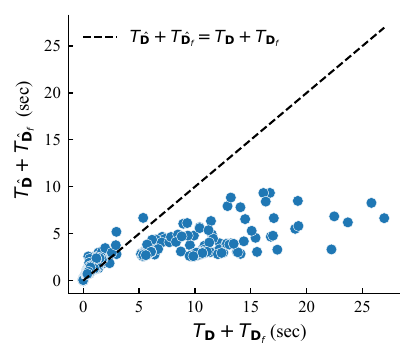}}
\hspace{3mm}
\subfloat[]{\label{subfig:approx,d}
\includegraphics[width=3.6cm]{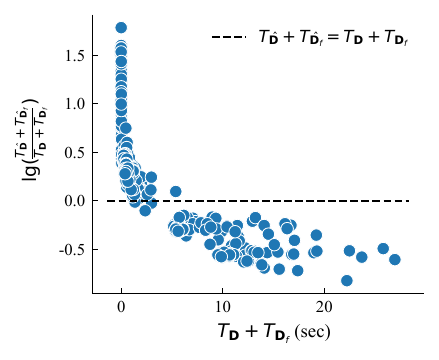}}
\vspace{-2.5mm}\caption{%
Comparison of approximation distances between sets with precise distances that are calculated directly by definition, evaluated on test data. 
(a) Scatter plot showing approximated values and precise values of distances between sets; 
(b) Relative difference comparison of \ppsalgabbr{} with direct computation concerning distance values. 
(c--d) Comparison of time cost (second) between \ppsalgabbr{} and direct computation based on Eq.~\eqref{eq:5}. 
}\label{fig:approx}
\end{minipage}
\end{figure}

\begin{figure}[t]%
\begin{minipage}{\linewidth}
\centering
\subfloat[]{\label{subfig:pm,a}
\includegraphics[width=3.7cm]{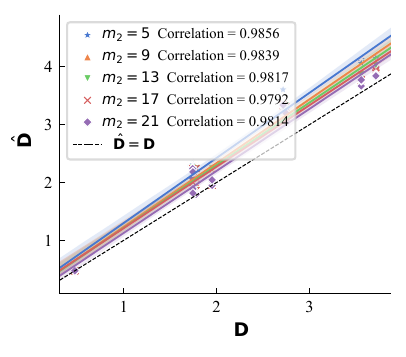}}
\subfloat[]{\label{subfig:pm,d}
\includegraphics[width=3.7cm]{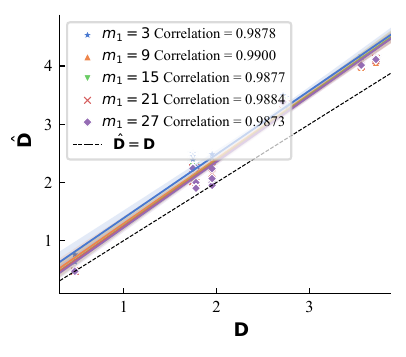}}
\\ \vspace{-4mm}
\subfloat[]{\label{subfig:pm,b}
\includegraphics[width=3.8cm]{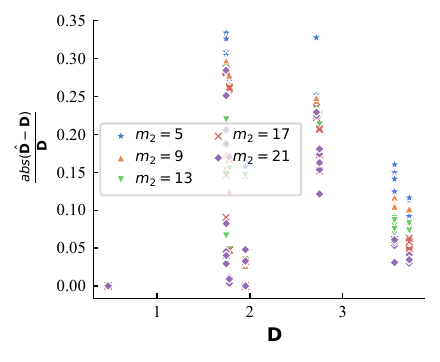}}
\subfloat[]{\label{subfig:pm,e}
\includegraphics[width=3.8cm]{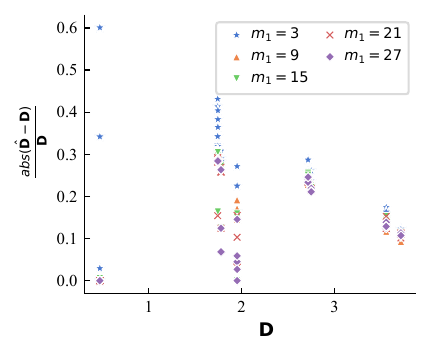}}
\\ \vspace{-4mm}
\subfloat[]{\label{subfig:pm,c}
\includegraphics[width=3.8cm]{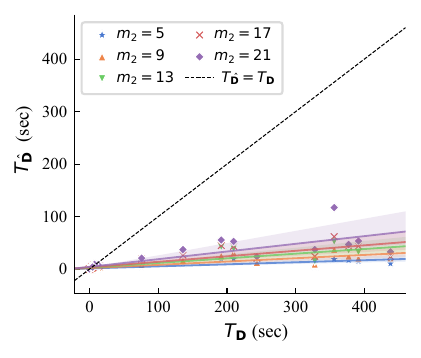}}
\subfloat[]{\label{subfig:pm,f}
\includegraphics[width=3.8cm]{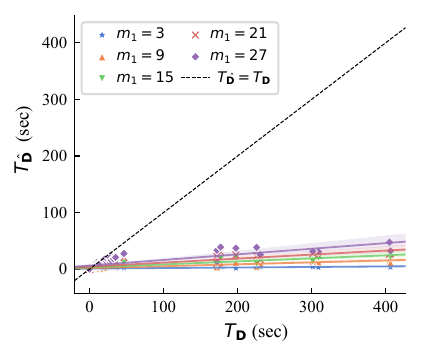}}
\vspace{-2.5mm}
\caption{Effects of hyperparameters $m_1$ and $m_2$ in \ppsalgabbr{}. 
(a) and (c), The effect of the hyperparameter $m_2$ on the distance value; (e) The effect of the hyperparameter $m_2$ on the time cost, where $m_1$ is set to 20. 
(b) and (d), The effect of the hyperparameter $m_1$ on the distance value; (f) The effect of the hyperparameter $m_1$ on the time cost, where $m_2$ is set to $\lceil2\lg~n\rceil$ in terms of $n$---the size of the corresponding dataset. 
}\label{fig:params}
\end{minipage}
\end{figure}
\begin{figure}[t]
\begin{minipage}{\linewidth}
\centering%
\subfloat[]{\includegraphics[height=31mm]{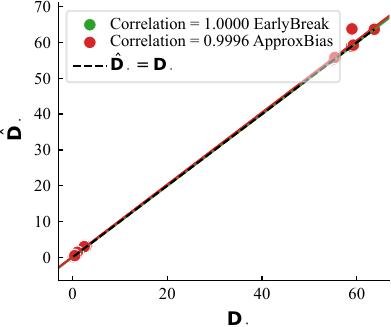}}
\subfloat[]{\includegraphics[height=31mm]{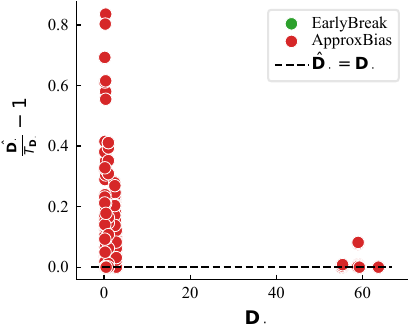}}
\\ \vspace{-4mm}
\subfloat[]{\includegraphics[height=31mm]{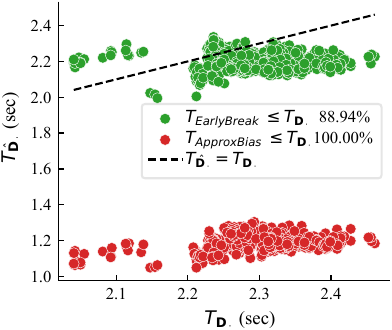}}
\subfloat[]{\includegraphics[height=31mm]{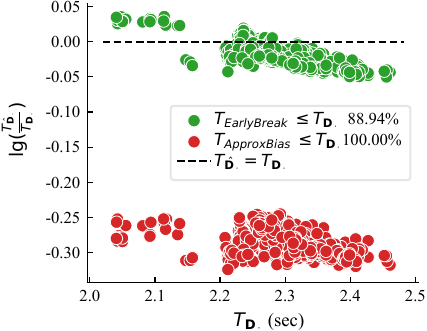}}
\vspace{-2.5mm}
\caption{Comparison between \ppsalgabbr{} and EarlyBreak \citep{taha2015efficient}, on the Income dataset. 
(a--b) shows whether the estimated values, generated by \ppsalgabbr{} or EarlyBreak, are accurate; (c--d) Time cost comparison, showing that \ppsalgabbr{} has higher computational efficiency than EarlyBreak. 
}\label{fig:earlybreak}
\end{minipage}
\end{figure}

\paragraph{Baseline methods}
To evaluate the validity of \ppsalgabbr{} in efficiently estimating the exact values of distances, we compare it with EarlyBreak \citep{taha2015efficient}, an efficient algorithm for calculating the exact distances. 
To evaluate the validity of \ppsdisabbr{} in capturing the discriminative degree of classifiers, we compare it with 
three commonly-used group fairness measures (that is, demographic parity (DP) \citep{feldman2015certifying,gajane2017formalizing}, equality of opportunity (EO) \citep{hardt2016equality}, and predictive quality parity (PQP) \citep{chouldechova2017fair,verma2018fairness})\footnote{%
Three commonly used group fairness measures of one classifier $f(\cdot)$ are evaluated as
\begin{scriptsize}
\begin{subequations}
\topequation\scriptsize
\begin{align}
    \mathrm{DP}(f) &= \lvert
    \mathbb{P}_\mathcal{D}[ f(\bm{x})\!=\!1 | \xpos\!=\!1 ] -
    \mathbb{P}_\mathcal{D}[ f(\bm{x})\!=\!1 | \xpos\!=\!0 ]
    \rvert \,,\\
    \mathrm{EO}(f) &= \lvert
    \mathbb{P}_\mathcal{D}[\, f(\bm{x})\!=\!1 | \xpos\!=\!1,\, y\!=\!1 ] - 
    \mathbb{P}_\mathcal{D}[\, f(\bm{x})\!=\!1 | \xpos\!=\!0,\, y\!=\!1 ]
    \rvert ,\\ 
    \mathrm{PQP}(f) &= \lvert
    \mathbb{P}_\mathcal{D}[\, y\!=\!1 | \xpos\!=\!1,\, f(\bm{x})\!=\!1 ] -
    \mathbb{P}_\mathcal{D}[\, y\!=\!1 | \xpos\!=\!0,\, f(\bm{x})\!=\!1 ]
    \rvert ,
\end{align}%
\end{subequations}%
\end{scriptsize}%
respectively, where $\insx= (\xneg,\xpos)$, 
$y$, and $f(\bm{x})$ are respectively the features, the true label, and the prediction of this classifier for one instance. Note that $\xpos=1$ and $0$ respectively mean that the instance $\insx$ belongs to the privileged and the marginalised groups.  
}, two individual fairness measures (that is, general entropy indices (GEI) \citep{speicher2018unified} and the Theil index (Theil) \citep{haas2019price})\footnote{%
For a constant $\gamma\notin\{0,1\}$, the generalised entropy indices for a problem with $n$ instances are defined, to quantify algorithmic unfairness, as
\begin{equation}
\topequation\scriptsize\textstyle
    \mathrm{GEI}^\gamma = 
    \frac{1}{n\gamma(\gamma-1)}
    \sum_{i=1}^n \Big(
        \big(\tfrac{b_i}{\mu}\big)^\gamma
    -1 \Big) \,,
\end{equation}
where benefits $b_i= f(\xneg_i,a_{1i})-y_i+1$ and $\mu=\sfrac{\sum_i b_i}{n}$. 
The Theil index is a special case for $\gamma=1$, that is, 
\begin{equation}
\topequation\scriptsize\textstyle
    \mathrm{Theil}= 
    \frac{1}{n}\sum_{i=1}^n
    \frac{b_i}{\mu} \log
    \big(\frac{b_i}{\mu}\big)
    \,.
\end{equation}
They are used additionally to group fairness measures to compare different algorithms and determine which one is considered the fairest from an individual perspective. 
}, and discriminative risk (DR)\footnote{%
The discriminative risk (DR) of this classifier is evaluated as
\begin{equation}
\topequation%
    \scriptsize
    \mathrm{DR}(f) = \mathbb{E}_\mathcal{D}[ 
        \mathbb{I}(f(\xneg,\xpos)\neq f(\xneg,\xqtb)) 
    ] \,,
\end{equation}%
where $\xqtb$ represents the disturbed sensitive attributes. DR reflects the bias degree of one classifier from both individual and group fairness aspects. 
} \citep{bian2023increasing_alt}. 
Note that only DR can take both individual and group fairness aspects into account at the same time.
Also note that BENN \citep{giloni2022benn} and the FAE-based metric \citep{wang2024procedural} are not included because they require extra trained models for unfairness evaluation, and therefore, they would break the same conditions in a comparison.

\paragraph{Implementation details} 

We mainly use bagging, AdaBoost, LightGBM \citep{ke2017lightgbm}, FairGBM \citep{cruz2022fairgbm}, and AdaFair \citep{iosifidis2019adafair} as learning algorithms, where FairGBM and AdaFair are two fairness-aware ensemble-based methods. 
Plus, certain kinds of classifiers are used in Section~\ref{subsec:RQ1}---including decision trees (DT), naive Bayesian (NB) classifiers, $k$-nearest neighbours (KNN) classifiers, Logistic Regression (LR), support vector machines (SVM), linear SVMs (linSVM), and multilayer perceptrons (MLP)---so that we have a larger learner pool to choose from based on different fairness-relevant rules. 
Standard 5-fold cross-validation is used in these experiments; in other words, in each iteration, the entire dataset is divided into two parts, with 80\% as the training set and 20\% as the test set. 
Also, features of datasets are scaled in preprocessing to lie between 0 and 1. 
Except for the experiments for \textbf{RQ3}, we set the hyperparameters $m_1=25$ and $m_2= \lceil 2\mathrm{lg}~n\rceil$ in other experiments, where $n$ is the size of the corresponding dataset. 
The value of 25 for $m_1$ is randomly chosen, which is neither too large to incur high computational costs, nor too small to lead to worse approximation effectiveness.

\subsection{Comparison between \ppsdisabbr{} and group fairness measures}
\label{subsec:RQ1}
In this experiment, we evaluate the performance of the proposed \ppsdisabbr{} compared with baseline group fairness measures. 

As groundtruth discriminative levels of classifiers remain unknown, and it is hard to directly compare different methods from that perspective, we compare the Pearson correlation coefficient between the performance and different fairness measures. 
The empirical results are reported in Figure~\ref{fig:delta,resbm}. 
Note that the reported \ppsdisabbr{} includes both $\newfist$ and $\hat{\newfist}$, where $\hat{\newfist}$ is the approximated value of the precise $\newfist$ via \ppsalgabbr{}. 
While \ppsdisabbr{} does not demonstrate a strong correlation with evaluation performance, its correlation is better than EO and PQP, shown in Figure~\ref{fig:delta,resbm}. 
The reason might be that \ppsdisabbr{} is only used to measure the extra bias introduced by classifiers or in the learning process, unlike other fairness measures that do not distinguish extra bias from overall discrimination within.

We also compare the test performance of the classifier that is chosen on training data based on fairness-relevant rules. 
In other words, we use a rule of `$\beta\!\cdot\!(1\!-\text{performance})\!+\!(1\!-\!\beta)\!\cdot\text{fairness}$' to select the best classifier and then report the corresponding test performance of these selected classifiers. 
The empirical results are reported in  Figure~\ref{fig:comp}. 
We observe from Figures~\ref{fig:comp}\subref{subfig:comp,a} and \ref{fig:comp}\subref{subfig:comp,b} that using \ppsdisabbr{} (\ie{} $\newfist$ and $\hat{\newfist}$) achieves the first and second best of the test accuracy over using three group fairness measures and DR, and using DR also achieves relatively good performance, 
which means it is indeed useful to consider the discriminative level of classifiers from both individual- and group-fairness aspects, while DP, EO, and PQP only consider group-fairness aspects. 
It also demonstrates that \ppsdisabbr{} could measure the discriminative level of classifiers somehow from both individual- and group-fairness perspectives, just like DR does. 
Using \ppsdisabbr{} also demonstrates superior performance in Figures~\ref{fig:comp}\subref{subfig:comp,alt,a} and \ref{fig:comp}\subref{subfig:comp,alt,b}. 
As for Figures~\ref{fig:comp}\subref{subfig:comp,d} and \ref{fig:comp}\subref{subfig:comp,alt,d}, using \ppsdisabbr{} (\ie{} $\hat{\newfist}$) presents relatively weaker fairness performance, especially when compared with DR. 
This is understandable because, unlike all other baseline fairness measures, \ppsdisabbr{} only measures the newly-introduced bias by classifiers in the learning, without reflecting the bias within the data. 
In other words, \ppsdisabbr{} is unlikely to recognise the primitive bias hidden in the data if no extra bias is introduced by the learner/classifier. 
Thus, \ppsdisabbr{} probably reflects partial bias information in the learning compared with other fairness measures, leading to relatively inferior performance. 
However, it is worth noting that using \ppsdisabbr{} shows pretty good results when evaluated from the perspective of PQP in Figures~\ref{fig:comp}\subref{subfig:fair,k2} and \ref{fig:comp}\subref{subfig:falt,k2}. 
Except for them, similar observations could also be found in other sub-figures of Figure~\ref{fig:comp}.\looseness=-1

\subsection{Comparison between \ppsdisabbr{} and individual fairness metrics}
In this experiment, we evaluate the performance of \ppsdisabbr{} in comparison with both group and individual fairness, with empirical results presented in Figures~\ref{fig:radar} and \ref{fig:tabular}. 
As we can see from the results on PPR and PPVR datasets (third to fourth rows) in Figure~\ref{fig:radar}, AdaFair (fourth to fifth columns) 
presents good performance concerning group fairness (especially, DP and EO) yet still exhibits bias from the perspective of individual fairness (GEI and Theil), and this phenomenon is indicated in \ppsdisabbr{}, suggesting that discrimination is not fully eliminated through AdaFair. 
Meanwhile, both \ppsdisabbr{} and group fairness indicate the existence of slight bias in AdaBoost and LightGBM (second and third columns). 
Considering that \ppsdisabbr{} can only capture the extra bias introduced in the learning procedure, we believe \ppsdisabbr{} is already reflecting both group- and individual-fairness aspects in its own way as best as it can. 
Furthermore, we particularly notice that, in Fig.~\ref{fig:tabular}\subref{subfig:tab4} and \ref{fig:tabular}\subref{subfig:tab5}, both DP and EO equal zero, and $\newfist(f)<0$, 
indicating that AdaFair mitigates bias in learning, yet bias is still indicated concerning PQP, GEI, and Theil. This is an interesting observation, suggesting that there is still room for AdaFair to improve and further mitigate bias hidden in data.

\subsection{Validity of \ppsalgabbr{} to approximate distances between sets for Euclidean spaces}
\label{subsec:RQ2}

In this experiment, we evaluate the performance of the proposed \ppsalgabbr{} compared with the precise distance that is computed by definition directly. 
To verify whether \ppsalgabbr{} could achieve the true distance between sets accurately and in a timely manner, we employ scatter plots to compare their values and time cost, presented in Fig.~\ref{fig:approx} to \ref{fig:params} (and Fig.~\ref{fig:approx,cont} of the Appendix). 
Particularly, we compare \ppsalgabbr{} with EarlyBreak \citep{taha2015efficient}, an efficient algorithm for calculating the exact distances, in Fig.~\ref{fig:earlybreak} and \ref{fig:earlybreak,cont} of the Appendix. From them, we observe that the estimated values produced by \ppsalgabbr{} are close to the corresponding exact values of distances, and that consume much less time than those produced by EarlyBreak. This suggests that \ppsalgabbr{} has higher computational efficiency.\looseness=-1

Besides, as we can see from Figure~\ref{fig:approx}\subref{subfig:approx,a}, the approximated values of distance using \ppsalgabbr{} are highly correlated with their correspondingly precise values; 
Besides, their linear fit line and the identity line (that is, $f(x)\!=\!x$) are near and almost parallel. 
Both observations mean that the approximation values of distance via \ppsalgabbr{} are valid and acceptable. 
Figure~\ref{fig:approx}\subref{subfig:approx,c} also presents that the relative difference between the approximation values and direct computation of distances is comparatively small and acceptable. 
It also shows that the relative difference between them becomes smaller as the distance $\newDist_\cdot$ increases. 
As for the time cost of \ppsalgabbr{} shown in Figure~\ref{fig:approx}\subref{subfig:approx,b}, all approximated values of distance cost less time than their direct computation of precise distance, and we find in practice that the time saving is usually more satisfactory when the used dataset is pretty large, such as on the Income dataset. 
Similar observations could also be found in Figures~\ref{fig:approx}\subref{subfig:approx,d}, \ref{fig:params}\subref{subfig:pm,c}, \ref{fig:params}\subref{subfig:pm,f}, \ref{fig:approx,cont}\subref{subfig:timecost,a}, and \ref{fig:approx,cont}\subref{subfig:timecost,c}, 
where in most cases of large datasets that need a longer time to proceed with the direct computation of distances, the time cost of \ppsalgabbr{} would not be worse than that of direct computation; yet in smaller dataset cases, the time cost of direct computation is less and completely acceptable, and using \ppsalgabbr{} would be less necessary.\looseness=-1

As for Figure~\ref{fig:approx,cont}, it is a more detailed version of Figure~\ref{fig:approx} by separating $\newDist_f$ and $\newDist$ from $\newDist_\cdot$ in Eq.~\eqref{eq:5}. Briefly speaking, Figures~\ref{fig:approx,cont}\subref{subfig:4cont,a} and \ref{fig:approx,cont}\subref{subfig:4cont,c} exhibit the same observation as in Figure~\ref{fig:approx}\subref{subfig:approx,a}; 
Figures~\ref{fig:approx,cont}\subref{subfig:4cont,b} and \ref{fig:approx,cont}\subref{subfig:4cont,d} exhibit the same observation at in Figure~\ref{fig:approx}\subref{subfig:approx,c}; 
Figures~\ref{fig:approx,cont}\subref{subfig:timecost,a} and \ref{fig:approx,cont}\subref{subfig:timecost,c} exhibit the same observation as in Figure~\ref{fig:approx}\subref{subfig:approx,b}; 
Figures~\ref{fig:approx,cont}\subref{subfig:timecost,b} and \ref{fig:approx,cont}\subref{subfig:timecost,d} exhibit the same observation as in Figure~\ref{fig:approx}\subref{subfig:approx,d}. 
To summarise, all these figures show that the approximated distances produced by \ppsalgabbr{} closely match the exact distances from direct computation, despite the use of random projections. These empirical evidences also suggest that \ppsalgabbr{} is robust to the randomness introduced by projection vectors and does not exhibit unstable behaviour across different runs.

\subsection{Effect of hyperparameters $m_1$ and $m_2$}
\label{subsec:RQ3}

In this subsection, we investigate whether different choices of hyperparameters (that is, $m_1$ and $m_2$) would affect the performance of \ppsalgabbr{} or not. 
Different $m_2$ values are tested when $m_1$ is set, and vice versa, with empirical results presented in Figure~\ref{fig:params}. 
As we can see from Figures~\ref{fig:params}\subref{subfig:pm,c} and \ref{fig:params}\subref{subfig:pm,f}, obtaining approximated values of distance distinctly costs less time than that of precise values, while increasing $m_2$ (or $m_1$) would cost more time but still be less than that of precise values. 
As for the approximation performance shown in Figures~\ref{fig:params}\subref{subfig:pm,a} and \ref{fig:params}\subref{subfig:pm,d}, all approximated values are highly correlated to the precise values of distance no matter how small $m_1$ or $m_2$ is, which means the effect of less proper choices of hyperparameters is relatively unapparent; As $m_1$ (or $m_2$) increases, the approximated values would become closer to the precise values of distance. 
This means that the approximation quality remains stable over a reasonably wide range of hyperparameter values, and that the risk associated with suboptimal hyperparameter selection is low in practice.

\begin{figure*}[t]
\begin{minipage}{\textwidth}
\centering%
\subfloat[]{\label{subfig:4cont,a}%
\includegraphics[height=30mm]{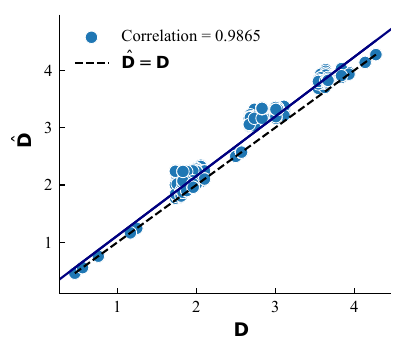}}
\subfloat[]{\label{subfig:4cont,b}%
\includegraphics[height=30mm]{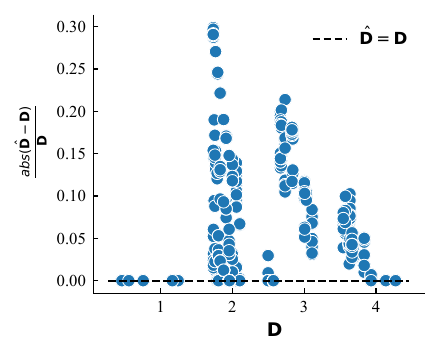}}
\subfloat[]{\label{subfig:4cont,c}%
\includegraphics[height=30mm]{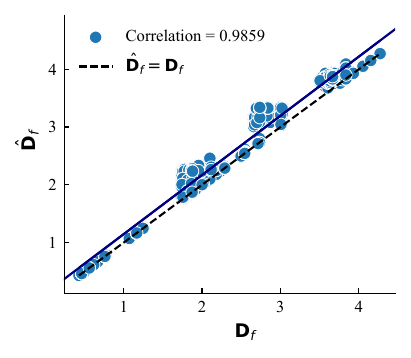}}
\subfloat[]{\label{subfig:4cont,d}%
\includegraphics[height=30mm]{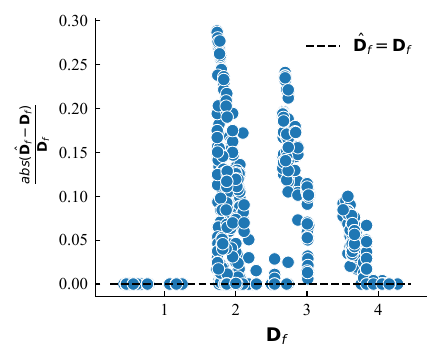}}
\\ \vspace{-.3cm}
\subfloat[]{\label{subfig:timecost,a}
\includegraphics[height=30mm]{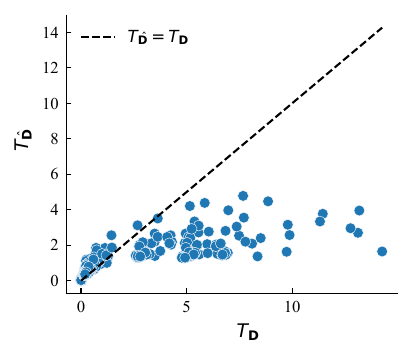}}
\subfloat[]{\label{subfig:timecost,b}
\includegraphics[height=30mm]{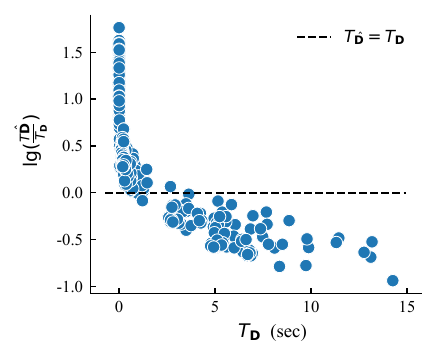}}
\subfloat[]{\label{subfig:timecost,c}
\includegraphics[height=30mm]{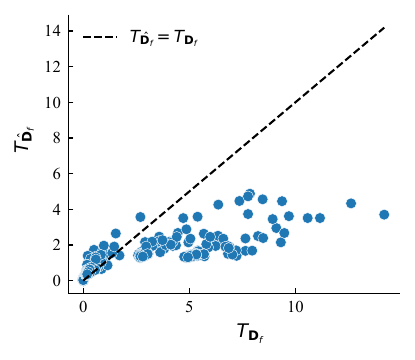}}
\subfloat[]{\label{subfig:timecost,d}
\includegraphics[height=30mm]{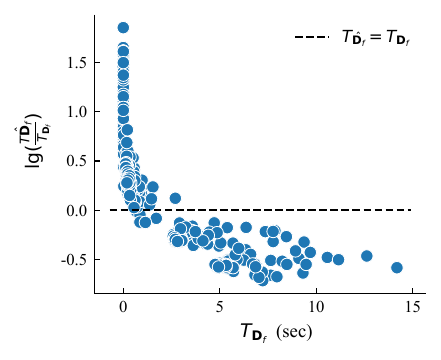}}
\vspace{-2mm}
\caption{Continuation of Figure~\ref{fig:approx}: 
Comparison of approximation distances using \ppsalgabbr{} with precise distances between sets that are calculated directly by definition, evaluated on test data. 
(a) Scatter plot showing approximated value $\hat{\newDist}$ and precise value $\newDist$ of distance between sets in Eq.~\eqref{eq:1}, and 
(b) Relative difference comparison between $\hat{\newDist}$ and $\newDist$; 
(c) Scatter plot showing approximated value $\hat{\newDist}_f$ and precise value $\newDist_f$ of distance between sets in Eq.~\eqref{eq:2}, and 
(d) Relative difference comparison between $\hat{\newDist}_f$ and $\newDist_f$; 
(e--f) Comparison of time cost between getting $\hat{\newDist}$ and $\newDist$, and (g--h) Comparison of time cost between getting $\hat{\newDist}_f$ and $\newDist_f$. 
Note that Figure~\protect\myref{fig:approx}{subfig:approx,a} is divided into (a) and (c), Figure~\protect\myref{fig:approx}{subfig:approx,c} is divided into (b) and (d), and Figures~\protect\myref{fig:approx}{subfig:approx,b} and \protect\myref{fig:approx}{subfig:approx,d} are divided analogously.\looseness=-1 
}\label{fig:approx,cont}
\end{minipage}
%
\begin{minipage}{\textwidth}
\vspace{-1mm}
\centering
\subfloat[]{\includegraphics[height=31mm]{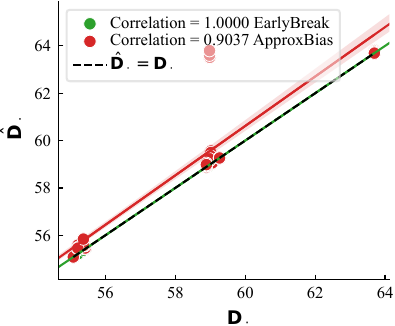}}
\hspace{2mm}
\subfloat[]{\includegraphics[height=31mm]{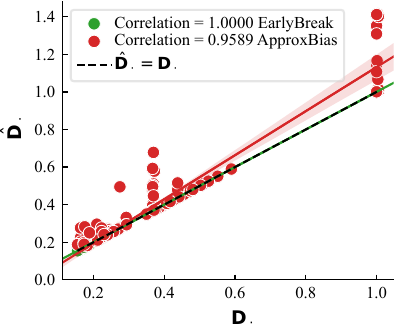}}
\hspace{2mm}
\subfloat[]{\includegraphics[height=31mm]{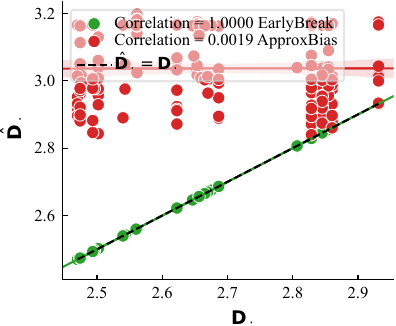}}
\hspace{2mm}
\subfloat[]{\includegraphics[height=31mm]{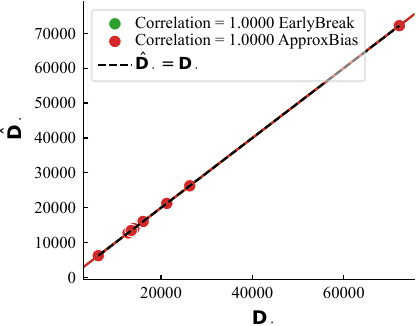}}
\\ \vspace{-3mm}
\subfloat[]{\includegraphics[height=31mm]{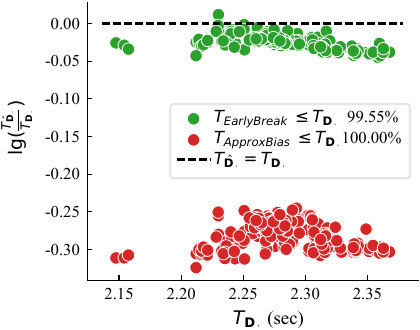}}\hspace{2mm}
\subfloat[]{\includegraphics[height=31mm]{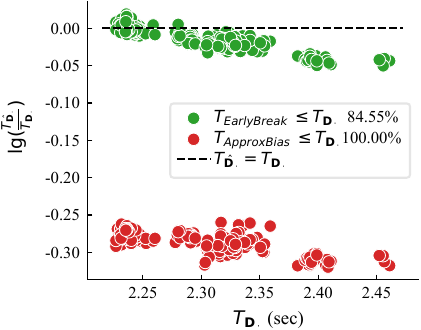}}\hspace{2mm}
\subfloat[]{\includegraphics[height=31mm]{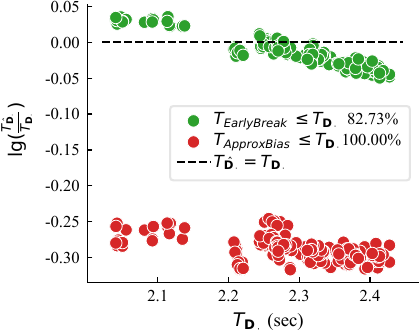}}\hspace{2mm}
\subfloat[]{\includegraphics[height=31mm]{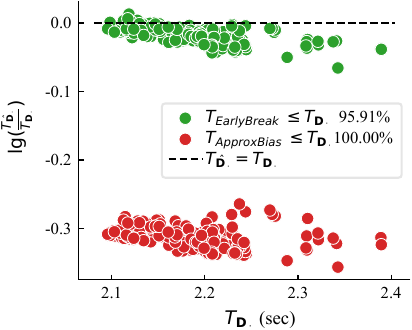}}
\vspace{-2mm}
\caption{Continuation of Figure~\ref{fig:earlybreak}: 
Four columns correspond to different preprocessing methods (standardization, normalization, min-max scaling, and no preprocessing, respectively). 
(a--d) Comparison between the estimated values and exact values of $\newDist_\cdot$; (e--h) Computational efficiency comparison.
}\label{fig:earlybreak,cont}
\end{minipage}
\end{figure*}

\section{Conclusion}
In this paper, we investigate how to measure the discrimination level of classifiers from both individual and group fairness aspects and present a novel harmonic fairness measure (\ppsdisabbr{}) based on distances between sets. 
To facilitate the calculation of distances and reduce its time cost from $O(n^2)$ to $O(n\mathrm{log}~n)$, we propose an approximation algorithm \ppsalgabbr{} and further present its algorithmic effectiveness analysis. 
The empirical results have demonstrated that the proposed \ppsdisabbr{} and \ppsalgabbr{} are valid and effective. 
However, the limitation of \ppsalgabbr{} is that it may not converge and thus can only reach approximated values of the distance. 
\ppsdisabbr{} also has a limitation that it can only evaluate the extra bias introduced in the learning procedure. In other words, if the learning algorithm is purely ``fair'' and no extra bias is introduced, it may not be able to detect whether discrimination exists in the data already. 
It would therefore be beneficial if some other fairness measures (such as GEI, Theil, and DR) are also considered in model evaluation, accompanying \ppsdisabbr{}. 
In practice, we indeed observed many zeros and even negative values, which leaves an open question about how to interpret them in such cases. In the future, we plan to explore the possibility of a fuller bias evaluation in this direction. To exploit it for improving learning algorithms could be interesting as well. 
\looseness=-1

\section*{Acknowledgments}

This research is funded by the European Union (MSCA, FairML, project no. 101106768). 

Views and opinions expressed are those of the author(s) only and do not necessarily reflect those of the European Union or the Research Executive Agency. Neither the European Union nor the granting authority can be held responsible for them.


{\appendices
\section{Additional Empirical Results}
\label{sec:addtl}
In this section, we list additional empirical results of Section~\ref{subsec:RQ2} that are left out in the main paper to save space, that is, Figures~\ref{fig:approx,cont} and \ref{fig:earlybreak,cont}.

}

\bibliographystyle{IEEEtran}
\bibliography{nus_title_iso4,refsmac}

 





\end{document}